%% file: main.tex
\documentclass[final,12pt]{arxiv} 

\usepackage{my_commands}
\usepackage{caption}
\usepackage{enumitem}
\usepackage{booktabs}
\usepackage{makecell}
\setlist[itemize]{leftmargin=5.5mm} % Smaller margins to the left in itemize

\usepackage{subfiles} %For compiling subfiles into one latex file
\title{Interpolation and Regularization for Causal Learning}
\usepackage{times}
\coltauthor{%
 \Name{Leena Chennuru Vankadara*} \Email{leena.chennuru-vankadara@uni-tuebingen.de}\\
 \addr University of Tübingen
 \AND
 \Name{Luca Rendsburg*} \Email{luca.rendsburg@uni-tuebingen.de}\\
 \addr University of Tübingen%
 \AND
 \Name{Ulrike von Luxburg} \Email{ulrike.luxburg@uni-tuebingen.de}\\
 \addr University of Tübingen%
 \AND
 \Name{Debarghya Ghoshdastidar} \Email{ghoshdas@in.tum.de }\\
 \addr Technical University of Munich, \\ Department of Informatics, \\ Munich Data Science Institute%
}

\begin{document}

\maketitle

\begin{abstract}%
% Interventional data from causal models are often unavailable, in which case learning has to rely on observational data.
We study the problem of learning causal models from observational data through the lens of interpolation and its counterpart---regularization.
A large volume of recent theoretical as well as empirical work suggests that, in highly complex model classes, interpolating estimators can have good statistical generalization properties and can even be optimal for statistical learning. Motivated by an analogy between statistical and causal learning recently highlighted by \citet{Jan:2019}, we investigate whether interpolating estimators can also learn good causal models. To this end, we consider a simple linearly confounded model and derive precise asymptotics for the \textit{causal risk} of the min-norm interpolator and ridge-regularized regressors in the high-dimensional regime. 
Under the principle of independent causal mechanisms, a standard assumption in causal learning, we find that interpolators cannot be optimal and causal learning requires stronger regularization than statistical learning. 
This resolves a recent conjecture in \citet{Jan:2019}. Beyond this assumption, we find a larger range of behavior that can be precisely characterized with a new measure of \textit{confounding strength}. If the confounding strength is negative, causal learning requires weaker regularization than statistical learning, interpolators can be optimal, and the optimal regularization can even be negative.
If the confounding strength is large, the optimal regularization is infinite and learning from observational data is actively harmful.
\end{abstract}


%%%%%%%%%%%%%%%%%%%%%%%%%%%%%%%%%%%%%%%%%%%%%%%%%%%%%%%%%%%%%%%%%%%%%%%%%%%%%%%%%%%%%%%%%%%%%%%%%%%%%
\section{Introduction}\label{sec:intro}
We consider the problem of learning the causal relationship between multivariate covariates $x \in \R^d$ and a scalar target variable $y \in \R$ purely from observational data and possibly under the presence of hidden confounders. Formally, given finite samples $\mycurls{(x_i, y_i)}_{i=1}^n$ drawn independently and identically (i.i.d) from the joint \textit{observational distribution} $p(x,y) = p(x)p(y\vert x)$, the goal of causal learning is to predict the effects on the target variable $y$ under \textit{interventions} on the covariates $x$. In other words, using Pearl's notation \citep{pearl2009causal} for $do$ interventions, the goal is to learn a predictive model that minimizes the expected loss on a random draw from the \textit{interventional distribution} $p_{do}(x,y) = p(x)p(y \vert do(x))$, which can be different from the observational distribution.

Recently, \citet{Jan:2019} established a close analogy between statistical learning and \textit{causal learning} (albeit under a highly constructed confounded model). As a consequence, \citet{Jan:2019} suggested that under certain assumptions, standard statistical learning-theoretic techniques (such as norm-based regularization) typically suggested for optimal statistical generalization may also help learn good causal models. However, the classical statistical principles of bias-variance trade-off have been challenged in the recent years by highly complex classes of models that are trained to interpolate the data and yet achieve remarkable generalization properties across a broad range of problem domains \citep{zhang2021understanding}. A large volume of recent work suggests that interpolation can be compatible with and may even be necessary to achieve optimal statistical generalization in the high-dimensional regime \citep{belkin2018understand, belkin2019does, liang2020just, feldman2020does}. Despite the surge in interest, causal properties of such interpolating estimators have not yet been explored. In this work, we consider a simple linear causal model in the high-dimensional regime ($n, d \rightarrow \infty, d/n \in \mathcal{O}(1)$) and ask: can interpolating estimators achieve good causal generalization?
\subsection{Motivation and Related Work}
\paragraph{Resemblance between statistical and causal generalization}
The problem of causal learning can be regarded as an instance of the general problem of learning under distribution shifts---where the training (observational) distribution is shifted from the test (interventional) distribution. In the framework of out-of-distribution generalization, an interesting proposition for learning good causal models arises from the following high-level idea. The bias induced due to observing small sample sizes may be similar to the bias induced due to certain distribution shifts. Therefore, techniques for learning models with good \textit{out-of-sample} generalization performance (for example, regularized risk minimization) {may} also help learn models with good \textit{out-of-distribution} generalization and vice-versa. One can find plentiful evidence in literature to support this general principle for different classes of distribution shifts. For instance, under a broad class of distribution shifts, distributionally robust optimization has been shown to be equivalent to norm-based regularization \citep{xu2009robustness, shafieezadeh2015distributionally, gao2017distributional, shafieezadeh2019regularization, blanchet2019robust, kuhn2019wasserstein}. Analogously, distributionally robust optimization techniques are also employed for statistical learning under limited samples \citep{zhu2020distributionally}. Of particular relevance to our work is the recent work of \citet{Jan:2019}, which formally establishes a close analogy between ``generalizing from \textit{empirical to observational distributions}'' and ``generalizing from \textit{observational to interventional distributions}'' under a highly constructed confounding model. As a consequence, \citet{Jan:2019} suggests that under reasonable assumptions standard norm-based regularization such as lasso or ridge typically used for statistical learning may also help learn better causal models. 
\paragraph{Interpolation can be compatible with statistical learning} Explicit norm-based regularization techniques have basis in classical learning theory principles of bias-variance trade-off, which is characterized by the classical U-shaped generalization curve. This principle recommends to avoid interpolation and instead suggests to balance data fitting with the complexity of the hypothesis class. Recently, however, these classical principles have been challenged by deep learning models. Despite being highly complex with the ability to even fit random labels and often trained to interpolate the training data, they achieve state-of-the-art out-of-sample generalization performance across a broad range of domains \citep{zhang2021understanding}. A partial explanation has been provided by the \textit{double-descent} phenomenon \citep{belkin2019reconciling, belkin2021fit}. Extending the generalization curve beyond the interpolation threshold reveals two regimes: the classical U-curve in the \textit{underparameterized} regime and a monotonically decreasing curve in the \textit{overparameterized} regime. This behaviour has been observed in deep neural networks as well as in other, simpler settings, for example, random feature models and random forests \citep{belkin2019reconciling, Has:2019, Mei:2019}. Follow-up work suggests that in the overparameterized regime, interpolating estimators can indeed achieve low statistical risk \citep{belkin2019does, liang2020just, bartlett2020benign, tsigler2020benign, muthukumar2020harmless}.

\paragraph{Is interpolation compatible with causal learning?}
On account of the parallels between statistical (out-of-sample) learning and causal (out-of-distribution) learning, it is therefore natural to ask: \textit{can interpolating estimators also learn good causal models?} For general classes of distribution shifts, one line of empirical work suggests that naively applying distributionally robust learning techniques such as importance reweighting or distributionally robust optimization approaches (which are equivalent to certain forms of regularization) may offer vanishing benefits over empirical risk minimization in overparameterized model classes \citep{byrd2019effect, sagawa2020investigation, gulrajani2021in}. However, there is also empirical evidence that suggests that augmenting such techniques with additional explicit norm-based regularization may help in learning distributionally robust models in the overparameterized regime \citep{sagawa2020investigation, donhauser2021interpolation}. In the context of causal learning, \citet{Jan:2019, vankadara2021causal} suggest that explicit regularization may help improve causal generalization. Furthermore, \citet{Jan:2019} conjectures generally that one may need to regularize more strongly for causal learning than for statistical learning. Existing work does not systematically assess the role of explicit regularization in causal learning, or correspondingly, whether interpolation is compatible with causal learning. In this work, we take a theoretical approach to systematically address these questions.
\subsection{Our Contributions}
We provide a first analysis of causal generalization from observational data in the modern, overparameterized and interpolating regime under a simple linear causal model. Specifically, we consider the interpolating minimum $l_2$ norm least-squares estimator as well the family of regularized ridge regression estimators in the proportional asymptotic regime. Subject to our model assumptions, we seek answers to the following questions: under what conditions can the optimal causal regularization parameter be $0$ or even negative, that is, do we observe \textit{benign causal overfitting}? Furthermore, if the optimal causal regularization parameter is strictly positive, how strongly do we need to regularize? How does the optimal causal regularization relate to the optimal statistical regularization? 
While our analysis is exhaustive, we emphasize the results under the assumption of independent causal mechanisms \citep{Jan:2010}, a standard assumption in causal learning.
\begin{itemize}
    \item \textbf{Precise asymptotics of the causal risk (Section~\ref{sec:theory_results}).} We provide precise asymptotics of the \textit{causal risk} of the ridge regression estimator as well as the minimum $l_2$ norm interpolating estimator in the high-dimensional setting: $n,d \rightarrow \infty, d/n \rightarrow \gamma \in (0, \infty)$. Our results confirm that, similar to the statistical setting, the causal generalization curve of the min-norm estimator exhibits the double-descent phenomenon. This is because the variance term diverges at the interpolation threshold and is decreasing in the overparameterized regime ($\gamma > 1$).
    \item \textbf{A measure of confounding strength $\genConf$ (Section~\ref{subsec:confounding_strength}).} We introduce a new measure of \textit{confounding strength} $\genConf$ that measures the relative contribution of the ``confounding signal'' to the ``causal signal''. This measure $\genConf$ can be interpreted as the strength of the distribution shift between the observational and interventional distributions. Under the assumption of independent causal mechanisms this measure is restricted to $[0,1]$ and induces a strict, model-independent ordering of the family of causal models that entail the same observational distribution.
    % we show that this measure induces a strict ordering on the family of causal models that correspond to the same statistical model, independent of the 
    %
    \item \textbf{Benign causal overfitting (Section~\ref{sec:benign}).} 
    % We show that in settings where the causal signal dominates the statistical signal ($\genConf < 0$), for certain regions of the overparameterization ratio $\gamma$ in both the \textit{underparameterized} ($\gamma < 1$) as well as \textit{overparameterized} ($\gamma > 1$) regimes, the optimal regularization can indeed be $0$ or negative even when the optimal statistical parameter is strictly positive. 
    % \TODO{We derive precise conditions that determine when the optimal causal regularization is strictly positive as well as conditions under which the optimal causal regularization can be negative.}
    We show that when the causal signal dominates the statistical signal ($\genConf < 0$), the optimal causal regularization can indeed be $0$ or negative even if the optimal statistical regularization is strictly positive. 
    This can happen both in the \textit{underparameterized regime} ($\gamma < 1$) as well as the \textit{overparameterized regime} ($\gamma > 1$).
    The size of this region grows as the causal signal increasingly dominates the statistical signal. Under the assumption of independent causal mechanisms, however, we show that there is no benign causal overfitting. This is in contrast to the statistical setting where the optimal regularization can be $0$ in the highly underparameterized regime ($\gamma \rightarrow 0$).
    \item \textbf{Optimal causal vs.\ statistical regularization (Section~\ref{sec:optimal_regularization}).} We show that one needs to regularize less strongly for causal learning than for statistical learning when the confounding strength $\genConf$ is negative.
    However, when $\genConf > 0$ and in particular under the principle of independent causal mechanisms, we show that one always needs to regularize more strongly for causal than for statistical learning. This resolves a recent conjecture in \citet{Jan:2019}. Indeed, our results show something stronger: the optimal causal parameter is a strictly increasing function in confounding strength. That is, as the confounding strength increases, one needs to regularize increasingly strongly for causal generalization and when $\genConf \geq 1$, one needs to regularize infinitely more for causal than for statistical learning. 
    % Finally, based on our results, we propose a simple method to choosing the correct regularization parameter if the confounding strength measure is either known or can be estimated. 
\end{itemize}
%%%%%%%%%%%%%%%%%%%%%%%%%%%%%%%%%%%%%%%%%%%%%%%%%%%%%%%%%%%%%%%%%%%%%%%%%%%%%%%%%%%%%%%%%%%%%%%%%%%%%
\section{Problem Setup}\label{sec:setup}
% \leena{Motivate why we analyse the linear model instead of a more complex one? The general reason is that analyzing causality is rather complicated in general and that is why most causality analysis (cite dominiks papers here) works in the linear setting and all the existing datasets are also linear }
%
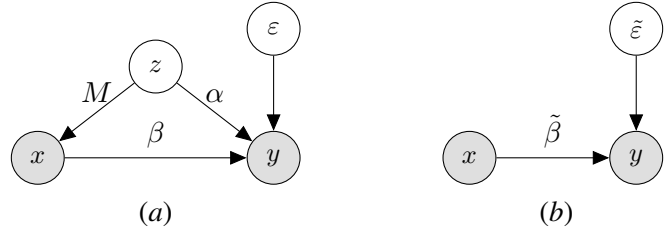
\begin{figure}[htbp]
\floatconts
    {fig:DAGs}
    {\caption{(\textit{a}) Graphical model of the causal model defined in \eqref{eq:causal_model}. (\textit{b}) The usual statistical model. In both figures, observed random variables are shaded and unobserved variables are white.}}
    {
        \subfigure{
        \label{fig:causal_DAG}
        \input{fig/causal_DAG}
        }
        \hspace{40pt} 
        \subfigure{
        \label{fig:statistical_DAG}
        \input{fig/statistical_DAG}
        }
     }
\end{figure}

We consider a linear causal model with parameters $M\in\R^{d\times l}$, $\alpha\in\R^l, \beta\in\R^d$ with $l\geq d$ and $\sigmacaus>0$ described via the \textit{structural equations}
\begin{align}\label{eq:causal_model}
    z\sim\Gauss{0}{I_l}\,, \quad
    \varepsilon &\sim \Gauss{0}{\sigmacaus}\,,\quad
    x= Mz\,,\quad
    y= x^T\beta + z^T\alpha + \varepsilon\,.
\end{align}
The covariates $x\in\R^d$ and the observation $y\in\R$ are \textit{confounded} through $z$, which follows a standard normal distribution on $\R^l$. 
This structure implies that $ \expec x  = 0$ and the covariance of $x$ is $\Sigma\coloneqq \Cov x = MM^T$. A graphical representation of this causal model is given in Figure~\ref{fig:causal_DAG}.
% Observational joint distribution
The observational joint distribution of this causal model is given by $p(x,y)=p(x)p(y|x)$, where $x\sim \Gauss{0}{\Sigma}$ and $y|x\sim\Gauss{x^T\betaStat}{\Epssigmastat}$.
Here, the statistical parameter $\betaStat\coloneqq\beta+\Gamma$ consists of the causal parameter $\beta$ and a confounding parameter $\Gamma\coloneqq\Sigma^+M\alpha$, and $\Epssigmastat\coloneqq\sigmacaus+\norm{\alpha}^2-\norm{\Gamma}_\Sigma^2$ describes the statistical noise,\footnote{Note that $\lnorm{\alpha}^2-\norm{\Gamma}_\Sigma^2=\lnorm{\alpha}_{I-M^+M}^2\geq 0$, where $I-M^+M$ describes the orthogonal projection onto $\ker M$.} where $\norm{x}^2_\Sigma\coloneqq x^T\Sigma x$ denotes the generalized norm.
Note that the observational distribution alone cannot distinguish the causal model from the one in Figure~\ref{fig:statistical_DAG}.
% Statistical learning
% The goal of statistical learning is to predict $y$ after observing $x$, the observational conditional $p(y|x)$ based on i.i.d. samples $\mycurls{(x_i, y_i)}_{i=1}^n$ from the observational joint distribution $p(x,y)$.
The goal of \emph{statistical learning} is to predict $y$ after observing $x$, which is captured by the conditional distribution $p(y|x)$.
% Causal learning
In contrast, the goal of \emph{causal learning} is to predict $y$ after manipulating or intervening on $x$. This is formally captured by Pearl's $do$-calculus \citep{pearl2009causal}, which describes how interventions on random variables introduce a shift to the joint distribution.
Graphically, intervening on $x$ with the value $x_0$, denoted as $do(x=x_0)$, removes all arrows to $x$ in the graphical model and \textit{sets} $x=x_0$. 
In our causal model \eqref{eq:causal_model}, the intervention $do(x=x_0)$ removes the arrow from $z$ to $x$ and yields the updated structural causal equations
\begin{align*}
    z\sim\Gauss{0}{I_l}\,, \quad
    \varepsilon &\sim \Gauss{0}{\sigmacaus}\,,\quad x=x_0 \,,\quad
    y= x_0^T\beta + z^T\alpha + \varepsilon\,.
\end{align*}
The corresponding distribution of $y$ after intervening on $x$ is therefore given by $y|do(x=x_0)\sim\mathcal{N}(x_0^T\beta,\Epssigmastat-\norm{\Gamma}_\Sigma^2)$.
Since arbitrary interventions can introduce arbitrary shifts in the distribution, we consider the natural class of interventions drawn from the observational marginal distribution on $x$. This yields the interventional joint distribution $p_{do}(x, y)=p(x)p(y|do(x))$ with the slight abuse of notation $do(x)$ in which the random variable $x$ and its value coincide.

\paragraph{Causal learning from observational data}
Assume we are given i.i.d.\ samples $\mycurls{(x_i, y_i)}_{i=1}^n$ from the observational joint distribution $p(x,y)$, which we collect in $X\in\R^{n\times d}$ and $Y\in\R^n$. 
The usual statistical learning aims for the observational conditional $p(y|x)$, which means that train and test distributions coincide. Causal learning aims for the interventional conditional $p(y|do(x))$, a distribution shift problem for which train and test distributions differ.
We define the corresponding \emph{causal risk} $\riskCtotal$ and \emph{statistical risk} $\riskStotal$ of any linear regressor $\betaGen\in\R^d$ under the squared loss as
\begin{align}
    \riskCtotal(\betaGen) \coloneqq \expec_{x}\expec_{y|do(x)}(x^T\betaGen-y)^2 \quad\text{and}
   \quad \riskStotal(\betaGen) \coloneqq \expec_{x}\expec_{y|x}(x^T\betaGen-y)^2\,.\label{eq:causal_statistical_risk}
\end{align}
Under the causal model in Eq.~\eqref{eq:causal_model}, the risks are characterized by the following proposition, which is proven in Appendix~\ref{app:risk_formulas}.
\begin{restatable}[Causal and Statistical Risk]{propositionRestatable}{RiskFormulas}\label{prop:risk_formulas}
For any $\betaGen\in\R^d$, the causal and statistical risks defined in Eq.~\eqref{eq:causal_statistical_risk} satisfy
\begin{align*}
      \riskCtotal(\betaGen) = \norm{\betaGen-\beta}_\Sigma^2 + \Epssigmastat+\norm{\Gamma}_\Sigma^2 \quad\text{and}\quad
      \riskStotal(\betaGen) = \norm{\betaGen-\betaStat}_\Sigma^2 + \Epssigmastat\,.
\end{align*}
\end{restatable}
Therefore, $\beta$ is the optimal causal parameter and $\betaStat$ is the optimal statistical parameter.
In the following, we simply refer to them as causal and statistical parameters.

%%%%%%%%%%%%%%%%%%%%%%%%%%%%%%%%%%%%%%%%%%%%%%%%%%%%
\subsection{A New Measure of Confounding Strength}\label{subsec:confounding_strength}
Since the interventional distribution generally differs from the observational distribution, we require a measure that quantifies how this shift influences causal learning from observational data.
\paragraph{Signal-to-noise ratios (SNRs)} Before we define our measure of confounding strength, we first define the statistical and causal signal-to-noise ratios which help to intuitively understand our confounding strength measure. Recall that every causal model entails a statistical model since the causal parameter $\beta$ and the confounding parameter $\Gamma$ jointly specify the statistical parameter $\betaStat=\beta+\Gamma$. The statistical SNR is defined as usual by $\SNRstat\coloneqq \norm{\betaStat}^2/\Epssigmastat$.
For the causal SNR, a natural notion would be $\norm{\beta}^2/(\Epssigmastat-\norm{\Gamma}_\Sigma^2)$ if the learning algorithm had access to data from the interventional distribution $y|do(x)\sim\mathcal{N}(x^T\beta,\Epssigmastat-\norm{\Gamma}^2_\Sigma)$; but since we are constrained to data from the observational conditional $y\vert x \sim \Gauss{x^T\betaStat}{\Epssigmastat}$, the corresponding causal SNR, which quantifies the hardness of the learning problem, needs to take this into consideration. Accordingly, we consider the causal SNR as the ratio of the alignment between the statistical and causal parameters and the variance of the observational conditional. Formally, we define it as \mbox{$\SNRcaus\coloneqq \scalprod{\beta}{\betaStat} / \Epssigmastat$}. In what follows, we therefore often refer to $\scalprod{\beta}{\betaStat}$ as the \textit{causal signal} and $\norm{\betaStat}^2$ as the \emph{statistical signal}. 
Correspondingly, we refer to $\scalprod{\betaStat - \beta}{\betaStat} = \scalprod{\Gamma}{\betaStat}$ as the \emph{confounding signal}, which is the alignment between the confounding parameter $\Gamma$ and the statistical parameter $\betaStat$.
For the reader's convenience, we summarized this terminology in Table \ref{tab:signal_notation}.
\setlength{\tabcolsep}{10pt}
\begin{table}[t]
\caption{Terminology for signals in our causal model \eqref{eq:causal_model}.}
\label{tab:signal_notation}
\vskip 0.15in
\begin{center}
\begin{small}
\begin{tabular}{ccc}
\toprule
Causal parameter $\beta$ & Confounding parameter $\Gamma$ & Statistical parameter $\betaStat=\beta+\Gamma$ \\
\midrule 
Causal signal $\scalprod{\beta}{\betaStat}$ & Confounding signal $\scalprod{\Gamma}{\betaStat}$ & Statistical signal $\norm{\betaStat}^2$ \\
\bottomrule
\end{tabular}
\end{small}
\end{center}
\vskip -0.1in
\end{table}
\paragraph{Confounding strength} 
Regression on observational data implicitly assumes that the interventional distribution coincides with the observational distribution, while it can be shifted in general.
To quantify the impact of this distribution shift on the corresponding causal risk, we introduce a new \emph{confounding strength measure} $\genConf$. It measures the relative contribution of the confounding signal to the statistical signal and is defined by
\begin{align}\label{eq:confounding_strength}
    \genConf \coloneqq
    \frac{\scalprod{\Gamma}{\betaStat}}{\scalprod{\Gamma}{\betaStat} + \scalprod{\beta}{\betaStat}}
    =\frac{\scalprod{\Gamma}{\betaStat}}{\norm{\betaStat}^2}\,.
\end{align}
While other notions of confounding strength are possible, we will see later that this definition 
is well-suited to capture the shift strength for causal learning from observational data.
Without further restrictions, $\genConf$ can take any value in $\R$. 
The different regimes of $\genConf$ can be intuitively understood in terms of the causal signal and its relationship to the statistical signal. This measure divides causal models into the following three regimes:
\begin{itemize}
    \item $\genConf \geq1$: the causal signal $\scalprod{\beta}{\betaStat}$ is non-positive, which implies that causal and statistical parameters are orthogonal or negatively aligned. Learning the statistical parameter is adversarial to causal learning.
    \item $0 < \genConf < 1$: causal and statistical parameters are positively aligned but the causal signal is weaker than the statistical signal $\norm{\betaStat}^2$, for example $\beta=\betaStat/2$.
    \item $\genConf \leq 0$: the causal signal dominates the statistical signal, for example $\beta=2\betaStat$.
    % \luca{The current phrasing of this intuition seems wrong/slightly misses the point}
\end{itemize}
The SNRs are related to the confounding strength measure via $\SNRcaus=(1-\genConf)\SNRstat$. In particular, the causal signal decreases as the confounding strength increases.

\paragraph{The regime $0\leq\genConf\leq 1$ is practically most relevant}
Causal learning often requires strong assumptions because causal models cannot be uniquely identified by their observational distribution. A standard assumption is the principle of independent causal mechanisms (ICM) \citep{Jan:2010, Lem:2012, Pet:2017}, which informally asserts that the causal mechanisms share no information.
In our causal model \eqref{eq:causal_model}, a corresponding assumption could be that the causal mechanisms $\beta$ and $\Gamma$ are drawn from rotationally invariant distributions. This implies that $\scalprod{\beta}{\Gamma}\to 0$ as $d\to\infty$, which in turn falls in the regime $0\leq\genConf\leq 1$. 
While our following analysis covers all possible causal models, we pay special attention to this regime because it might be of most practical relevance.
Note that for $\scalprod{\beta}{\Gamma}=0$, our measure of confounding strength coincides with the structural strength of confounding measure $\structConf=\norm{\Gamma}^2/(\norm{\Gamma}^2+\norm{\beta}^2)$ introduced by \citet{Jan:2017}. It measures the relative contribution of causal and confounding signal in terms of lengths rather than inner products with the statistical signal. 

\section{Causal and Statistical Risk of High-Dimensional Regression Models}\label{sec:theory_results}
We consider two linear regression models for learning causal models from observational data $X,Y$: min-norm interpolation and ridge regression.
% Min-norm
The \emph{min-norm interpolator} is the minimum $l_2$ norm solution to the least squares regression problem
\begin{equation}\label{eq:min_norm_optimization}
    \betaMinNorm(X,Y) \coloneqq \argmin \{\norm{\betaGen}_2: \betaGen \in \argmin \limits_{\betaGen \in \R^d} \norm{Y - X \betaGen}^2\}.
\end{equation}
 A closed form is given by 
$\betaMinNorm(X,Y)=(X^TX)^+X^TY$, where $A^+$ denotes the Moore-Penrose inverse of $A$.
% Ridge
For $\lambda>0$, the \emph{ridge regression estimator} solves the regularized least squares problem 
\begin{align}\label{eq:ridge_optimization}
    \betaRidge(X,Y) \coloneqq \argmin_{\betaGen\in\R^d}\frac{1}{n}\norm{Y - X \betaGen}^2+\lambda\lVert \betaGen\rVert^2\,,
\end{align}
which has the explicit solution $\betaRidge(X,Y)=(X^TX+n\lambda I_d)^{-1}X^TY$. 
The min-norm interpolator can be obtained as a limiting case from the ridge regression solution via $\betaMinNorm(X,Y)=\lim_{\lambda\rightarrow 0^+}\betaRidge(X,Y)$. Whenever it is clear from the context, we drop the dependence of the predictors on $X$ and $Y$.
%%%%%%%%%%%%%%%%%%%%%%%%%%%%%%%%%%%%%%
\subsection{Precise Asymptotics of the Causal and Statistical Risks}\label{subsec:main_results}
In this section, we provide precise asymptotics for the causal and statistical risks of the min-norm interpolator and ridge regression solutions in the high-dimensional regime. This regime is characterized by both $n,d\to\infty$ such that $d/n\to\gamma \in (0, \infty)$, where $\gamma$ is called the \textit{overparameterization ratio}. We distinguish between the \textit{underparameterized regime} ($\gamma<1$) and the \textit{overparameterized regime} ($\gamma>1$).
% We present our main results in Section~\ref{subsec:main_results}, introduce a measure of confounding strength that serves as the basis for understanding the hardness of causal generalization in Section~\ref{subsec:confounding_strength}, and give a basic exposition of the results in Section~\ref{subsec:basic_exp}.
% We postpone the interpretation of the results to Section~\ref{sec:benign}, which investigates whether the min-norm interpolator can optimal, and Section~\ref{sec:optimal_regularization}, which investigates the optimal causal regularization parameter.
All proofs for this section are deferred to Appendix~\ref{app:causal_results}.
Since the predictors are random variables in the training data $X,Y$, so is their corresponding causal risk.
We consider the expectation of the risk under $Y$ conditioned on $X$. According to Proposition~\ref{prop:risk_formulas}, it is given by $\riskCGen\coloneqq\expec_{Y|X}\riskCtotal(\betaGen)=\expec_{Y|X}\norm{\betaGen-\beta}_\Sigma^2+\Epssigmastat+\norm{\Gamma}^2_\Sigma\,.$
Due to its simple form, similar to the usual statistical risk, the causal excess risk can be decomposed into bias and variance:
\begin{align}\label{eq:bv_decomp}
    \expec_{Y|X}\norm{\betaGen-\beta}_\Sigma^2
    =\underbrace{\norm{\expec_{Y|X}\betaRidge-\beta}^2_{\Sigma}}_{\eqqcolon\biasCRidge}  + \underbrace{ \expec_{Y \vert X} \norm{\betaRidge - \expec_{Y \vert X} \betaRidge}_{\Sigma}^2}_{\eqqcolon\varianceCRidge}\,.
\end{align}
The next theorem is our main result, which gives a closed-form expression for the limiting causal bias and variance of the min-norm interpolator and ridge regression estimators. 
We make the simplifying assumption of isotropic covariance $\Sigma=I_d$. Our proofs rely on results from random matrix theory following arguments similar to \citet{dicker2016ridge, Dob:2018, Has:2019}. They can similarly be extended to arbitrary covariances under boundedness assumptions on the spectrum. Since the isotropic causal model already exhibits rather rich behavior, we focus on thoroughly understanding this setting and leave such extensions for future work.

%%%% Theorem: causal min-norm
% \begin{restatable}[Limiting Causal Bias and Variance for the Min-Norm Interpolator]{theorem}{MinNormAsymptotics}\label{thm:causal_min_norm}Let $\norm{\beta}^2=\signalcaus$, $\norm{\Gamma}^2=\omega^2$, $\scalprod{\Gamma}{\beta} = \eta$, and $\sigma_{\tilde{\epsilon}}^2 = \sigmastat$. Then as $n,d\to\infty$ such that $d/n\to\gamma\in (0,\infty)$, it holds almost surely that
% \begin{align}
%     \biasCMinNorm\to
%     \limBiasArgC{0}=
%     \begin{cases}
%     \omega^2, &\gamma<1\\
%     \omega^2+ (\signalcaus-\omega^2)(1-\frac{1}{\gamma}), &\gamma>1
%     \end{cases}
%     \,,\quad
%     \varianceCMinNorm \to
%     \limVarianceArgC{0}=
%     \begin{cases}
%     \sigmastat\frac{\gamma}{1-\gamma}, &\gamma<1\\
%     \sigmastat\frac{1}{\gamma-1}, &\gamma>1
%     \end{cases}\,.
% \end{align}
% Therefore $\riskCMinNorm\to\limRiskArgC{0}=\limBiasArgC{0}+\limVarianceArgC{0}+\sigmastat+\omega^2$.
% %
% \end{restatable}
%%%% Theorem: causal ridge
\begin{restatable}[Limiting Causal Bias-Variance Decomposition for the Ridge Estimator]{theorem}{CausalRidge}\label{thm:causal_ridge}
Let $\norm{\beta}^2=\signalcaus$, $\norm{\Gamma}^2=\omega^2$, $\scalprod{\Gamma}{\beta} = \eta$, and $\sigma_{\tilde{\epsilon}}^2 = \sigmastat$. Then as $n,d\to\infty$ such that $d/n\to\gamma\in (0,\infty)$, it holds almost surely in $X$ for every $\lambda>0$ that
\begin{align}
    \biasCRidge&\to
    \limBiasArgC{\lambda}=\omega^2 +  \signalstat \lambda^2 m'(-\lambda)-  2 ( \omega^2 + \eta) \lambda m(-\lambda)\quad\text{and}\label{eq:limBiasCridge}\\
     \varianceCRidge&\to
     \limVarianceArgC{\lambda}=\sigmastat\gamma(m(-\lambda) - \lambda m'(-\lambda))\,,\label{eq:limVarianceCridge}
\end{align}
where 
% $m(\lambda) = \frac{(1 - \gamma - \lambda) -  \sqrt{(1  -  \gamma  -  \lambda)^2 - 4 \gamma \lambda}}{2 \gamma \lambda}$ 
$m(\lambda) = ((1 - \gamma - \lambda) -  \sqrt{(1  -  \gamma  -  \lambda)^2 - 4 \gamma \lambda}) / (2 \gamma \lambda)$ 
and $\signalstat = r^2 + \omega^2 + 2 \eta$.
Therefore $\riskCRidge\to\limRiskArgC{\lambda}=\limBiasArgC{\lambda}+\limVarianceArgC{\lambda}+\sigmastat+\omega^2$. 
The corresponding limiting quantities for the min-norm interpolator can be obtained by taking the limit $\lambda\to 0^+$ in \eqref{eq:limBiasCridge} and \eqref{eq:limVarianceCridge}.
\end{restatable}
From these limiting expressions we can see that the causal risk curve of the min-norm interpolator exhibits the double descent phenomenon: it diverges at the interpolation threshold $\gamma=1$ due to the variance term and decreases again for $\gamma>1$. A corresponding visualization is given in Figure~\ref{fig:causal_bv_decomp}.
Explicit regularization dampens the divergence of the variance term.

%%%%%%%%%%%%%%%%%%%%%%%%%%%%%%%%
While we are primarily interested in the causal risk, the corresponding statistical risk serves as a natural baseline. 
An analogue set of results for the statistical risk is given in Appendix~\ref{app:statistical_asymptotics}. 
These results have already been derived by \citet{Has:2019} and can also be recovered as a special case of our causal results: for fixed statistical parameters $\betaStat$ and $\Epssigmastat$, the statistical risk coincides with the causal risk of an unconfounded causal model defined with $\beta=\betaStat$, $\sigmacaus=\Epssigmastat$, and $\alpha=0$. In particular, the corresponding statistical limiting expressions are the same as in Theorem~\ref{thm:causal_ridge} after setting $\eta=\omega^2=0$.

%%%% Theorem: statistical ridge and min-norm
%
\paragraph{Optimal statistical and causal regularization} By directly optimizing the closed form expressions for limiting causal and statistical risks we can find the optimal causal and statistical regularization. For any $\gamma \in (0, \infty)$, the optimal statistical regularization $\lambdaS(\gamma) \coloneqq \arginf_{\lambda \in (0, \infty)} \limRiskArgS{\lambda}$ can be expressed in closed-form as $\lambdaS(\gamma)=\SNRstat^{-1}\gamma$. The closed-form expression for the optimal causal regularization parameter $\lambdaC(\gamma) \coloneqq \arginf_{\lambda \in (0, \infty)} \limRiskArgC{\lambda}$ is a root of a $4$th order polynomial and as such considerably intricate. For readability, we do not include it here. We investigate the behavior of the optimal causal and statistical regularization in Section~\ref{sec:benign}~and~\ref{sec:optimal_regularization}.
%%%%%%%%%%%%%%
\subsection{Basic Behavior of the Limiting Risk}\label{subsec:basic_exp}
\begin{figure}%[!htb]
    \centering
    \begin{minipage}[t]{.48\textwidth}
        \centering
        \floatconts
        {fig:expo_risk_min_norm}% label
        {\caption{Limiting causal excess risk $\limRiskArgC{0}$ (without the constant $\sigmastat+\omega^2$) of the min-norm interpolator for different causal signal strengths $S$. Dashed lines are the corresponding null-risks $\omega^2$, which are outperformed more often as $S$ increases. For $\gamma<1$, all three curves coincide.}}% caption command
        {\includegraphics{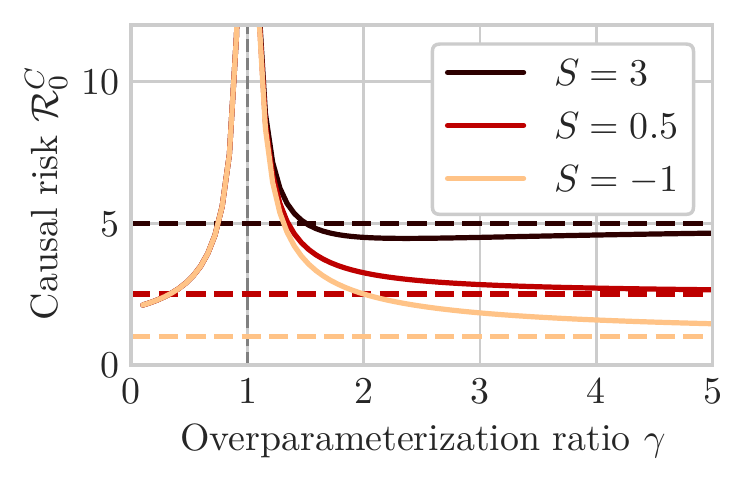}}
    \end{minipage}\hfill%
    \begin{minipage}[t]{0.48\textwidth}
        \centering
        \floatconts
        {fig:causal_bv_decomp}% label
        {\caption{Limiting bias-variance decomposition and excess risk of the min-norm interpolator (black) and optimally regularized ridge regression (red). Crosses indicate finite-sample risks of $n=d/\gamma$ samples with $d=300$, which are well-predicted by the limiting risk.}}% caption command
        {\includegraphics{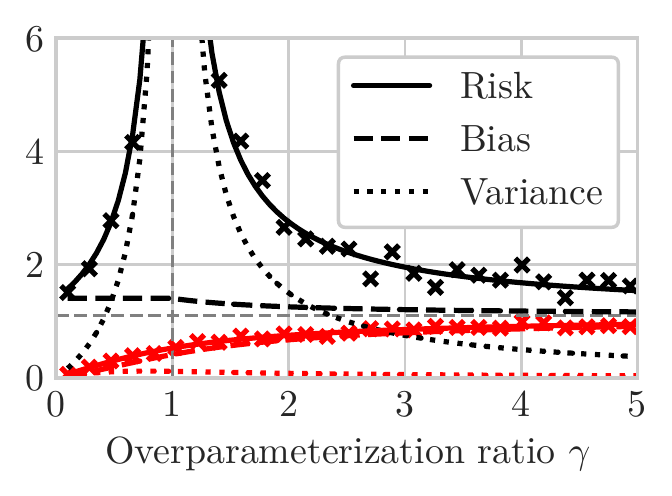}}
    \end{minipage}
\end{figure}
We start to analyze the results by assessing the basic behavior of the limiting causal risk. 
The causal risk of the null estimator $\betaGen=0$ serves as a natural baseline to evaluate the performance of the the min-norm interpolator and the ridge regression estimators.

% The null causal and statistical risks are defined as the correspo the null predictor $\betaGen=0$ are given by $\norm{\beta}^2 + \Epssigmastat - \norm{\Gamma}^2$ and $\norm{\betaStat}^2$ respectively. a natural baseline to evaluate the performance of the the min-norm interpolator and the ridge regression estimator.
% We focus on the min-norm interpolator, an extensive treatment of the general ridge-regression case is given in the following two sections. 

\paragraph{Regimes of the min-norm interpolator}
Theorem~\ref{thm:causal_ridge} characterizes the limiting causal risk of the min-norm interpolator.
Its behavior is controlled by the causal signal-to-noise ratio, which we defined as $\SNRcaus=(1-\zeta)\SNRstat$. However, as we will later see, the causal risk of the min-norm interpolator can be lower than null risk only when $\genConf < 0.5$. To distinguish the regimes of the min-norm estimator, its therefore convenient to consider the closely related quantity $S = (1 - 2 \genConf)\SNRstat$. It distinguishes between three different regimes (visualized in Figure~\ref{fig:expo_risk_min_norm}).
% in less confounded settings (up to $\genConf=.5$) than the optimally regularized ridge solution (up to $\genConf=1$), because it fits the data more strongly. It is therefore more convenient to consider the related quantity $\SNRminnorm=(1-2\zeta)\SNRstat=(r^2-\omega^2)/\sigmastat$ at this point. 
%
\begin{itemize}%[label=(\roman*)]
    \item For $\SNRminnorm>1$, the causal signal dominates the noise and the min-norm estimator can perform better than null risk in both under- and overparameterized regime.  
    \item For $0\leq\SNRminnorm\leq 1$, the causal signal is weaker than the noise. Only the underparameterized regime can beat the null risk, whereas the overparameterized regime is always worse.
    \item The previous two cases resemble the behavior of the statistical risk in the corresponding regimes of the statistical SNR. 
    Contrary to the statistical risk, however, the causal risk admits a third regime $\SNRminnorm<0$. 
    In this case, the min-norm estimator always performs worse than null risk.
    Here, the causal signal $\scalprod{\beta}{\betaStat}$ is dominated by the confounding signal $\scalprod{\Gamma}{\betaStat}$, and interpolating the observational data overfits to the confounding. 
\end{itemize}

\paragraph{Bias and variance} 
The bias-variance decomposition of the causal risk given in Theorem~\ref{thm:causal_ridge} is visualized in Figure~\ref{fig:causal_bv_decomp} for the min-norm interpolator and the optimally ridge-regularized regressor. 
The figure also shows the causal risk based on finite samples from the model, which is in high agreement with our asymptotic results.
We compare the causal risk to the corresponding statistical risk.
First note that the causal and statistical variance terms coincide exactly for both the min-norm estimator and ridge regressors.
This is because the variance term of the squared loss depends only on the variance in the training data, but not on the target parameter $\beta$ or $\betaStat$. Since the training data are the same for both causal and statistical learning, the variance terms trivially coincide.

For the min-norm estimator, as in the statistical case, the variance term causes the double-descent behavior of the causal risk curve because it explodes at the interpolation threshold $\gamma=1$ and is decreasing in the overparameterized regime $\gamma > 1$.
% Consequentially, all difference in behavior between causal and statistical risk is due to the bias terms.
In the statistical setting, the bias strictly increases in the overparameterized regime and as a consequence, the best risk is always achieved in the underparameterized setting. In contrast, the causal bias of the min-norm interpolator can be decreasing in the overparameterized regime and therefore the optimal causal risk can be achieved in the highly overparameterized regime $\gamma \rightarrow \infty$. However, this only happens in the regime $S < 0$ where the risk of the min-norm interpolator is always worse than null risk.
%
% While the statistical bias is an increasing function in $\gamma$, the causal bias of the min-norm interpolator can be decreasing in the overparameterized regime under strongly confounded models. However, this only happens when the min-norm interpolator is worse than null risk.

As shown in Figure \ref{fig:causal_bv_decomp}, the causal risk of the optimally regularized ridge regression estimator is always below that of the min-norm risk which is trivial. Similar to the statistical setting, the corresponding generalization curve does not exhibit the double descent phenomenon. There are qualitatively different reasons for why regularization helps in statistical and causal learning. 
For both statistical and causal learning, regularization decreases the shared variance, which corresponds to the finite-sample error. 
However, while the statistical bias always increases with regularization, the causal bias can actually decrease. This implies that regularization not only helps with the finite-sample error, but can also reduce the error due to confounding.

%%% Confounding strength
% First result: ordering of causal models
\paragraph{Higher confounding implies higher causal risk for all $\lambda$.}
So far, we have investigated the causal risk under a single causal model,
but we can compare different causal models using the confounding strength measure $\genConf$ introduced in Section~\ref{subsec:confounding_strength}.
The next proposition shows that $\genConf$ governs the hardness of causal learning from observational data. Specifically, the causal risk of the ridge regression estimator for any $\lambda \in (0, \infty)$ increases as the causal model becomes more confounded. A proof is given in Appendix~\ref{app:conf_strength_ordering}.  
\begin{restatable}[Causal Risk Increases with Confounding Strength]{propositionRestatable}{orderingCausal}\label{prop:conf_strength_ordering} Consider the family of causal models parameterized as in (\ref{eq:causal_model}) that entail the same observational distribution. 
Let $C_1$ and $C_2$ be two such causal models with confounding strengths $\genConf_1$ and $\genConf_2$ and alignments $\eta_1$ and $\eta_2$ (defined in Theorem~\ref{thm:causal_ridge}), respectively. Then for all $\lambda,\gamma \in (0, \infty)$, 
\begin{equation*}
    %  \label{eq:ordering_causal_gen}
     \genConf_1 > \genConf_2, \; \; \eta_1 \leq \eta_2 \implies \limRisk_\lambda^{C_1} > \limRisk_\lambda^{C_2}.
\end{equation*}
In particular, for any fixed $\eta$, the measure of confounding strength $\genConf$ establishes a strict ordering of causal models.
% , that is, $\limRiskC$ is increasing in $\genConf$. 
This includes the ICM under which $\eta=0$.
\end{restatable}
%
% Compare this result to \citet{Tri:2021} who consider a similar problem of learning under covariate shift. They define a model-independent notion of shift strength and show that harder shifts increase the risk.
% Proposition~\ref{prop:conf_strength_ordering} states that the confounding strength measure $\genConf$ plays the same role for causal learning as their measure does for covariate shift.
%%%%%%%%%%%%%%%%%%%%%%%%%%%%%%%%%%%%%%%%%%%%%%%%%%%%%%%%%%%%%%%%%%%%%%
\section{Benign Causal Overfitting}\label{sec:benign}
% It has been repeatedly documented now that even in highly complex models, optimal statistical generalization can be achieved when regularization is zero or even negative. This phenomenon is often referred to as benign overfitting or harmless interpolation. \TODO{Cite the literature}. This is prima
A large number of recent works suggest that the minimum-norm interpolating estimator can be optimal for statistical generalization \citep{belkin2018understand, belkin2019does, muthukumar2020harmless}. This phenomenon is often referred to as benign overfitting. Moreover, the optimal statistical generalization may even be achieved when regularization is negative \citep{kobak2020optimal, bartlett2020benign, tsigler2020benign}. It is unclear, however, if such interpolating estimators, which have implicit small-norm biases, can also be optimal when there is a shift between the training and test distributions. In particular, we ask: can optimal causal regularization be $0$ or even negative, that is, do we observe \textit{benign causal overfitting?} To show that the optimal regularization can be negative, we simply show that the derivative of the causal risk at $0$ is positive. We summarize our key findings in Theorem \ref{thm:negative_causal_reg}.
\begin{restatable}[Optimal Regularization can be Negative]{theorem}{NegativeCausalReg}
\label{thm:negative_causal_reg}
% \luca{Change formulation: make clear that the two cases here are for whether $\lambdaC\leq 0$ or not. Right now, it distinguishes between $\genConf$ cases, which is only part of the condition}
For any causal model parameterized as in \eqref{eq:causal_model}, the following cases distinguish between whether the min-norm interpolator is optimal or not.
\begin{enumerate}
    \item For negative confounding strength $\genConf<0$ the optimal causal regularization $\lambdaC$ can be $0$ or even negative. A necessary and sufficient condition for $\lambdaC \leq 0$ depends on the difference in causal and statistical signal-to-noise ratios and is given by
% \begin{equation*}
%     \SNRcaus - \SNRstat \geq \begin{cases}
%     \frac{\gamma}{(1 - \gamma)^2}, &\gamma < 1 \\
%     \frac{\gamma^2}{(1 - \gamma)^2}, &\gamma > 1
%     \end{cases}
% \end{equation*}
\begin{equation*}
    \SNRcaus - \SNRstat \geq\frac{\gamma \max{\{1,\gamma\}}}{(1 - \gamma)^2}\,.
\end{equation*}
\item However, when $\genConf>0$ the optimal causal regularization is strictly positive $\lambdaC > 0$ and $\limRiskArgC{0}>\limRiskArgC{\lambdaC}$, hence regularization provides non-vanishing benefits. This includes the ICM.
\end{enumerate} 
\end{restatable}
In the highly overparameterized regime $(\gamma \rightarrow \infty)$, the benefit of explicit regularization vanishes and both the causal and statistical risks of the ridge regression estimator converge to their corresponding null risks independent of the choice of regularization. We do not refer to this phenomenon as benign overfitting. 
The result is rather intuitive if interpreted via our measures of the causal and statistical SNRs and the confounding strength measure. When the causal SNR is larger than the statistical SNR ($\genConf < 0$) and for certain regions of the parameter space of $\gamma$, the optimal causal regularization can be zero or even negative. This phenomenon can be observed in both the underparameterized as well as the overparameterized regime. With increasing dominance of the causal signal over the statistical signal, the range of $\gamma$ for which the optimal causal regularization is negative increases. As $\gamma$ approaches the interpolation threshold, it becomes increasingly hard for the optimal causal regularization to be negative. Recall that the optimal statistical regularization can be expressed in closed form as $\lambdaS = \SNRstat^{-1} \gamma$ for any $\gamma \in (0, \infty)$ and therefore the optimal statistical regularization is always positive while the optimal causal regularization can be negative!

When the causal SNR is smaller than the statistical SNR ($\genConf > 0$) and in particular under the ICM ($0 < \genConf \leq 1$), the optimal causal regularization is strictly positive and the benefit of explicit regularization does not vanish. This can indeed be the case even when the optimal statistical regularization vanishes. To see this consider the statistical risk in the highly underparameterized regime $\gamma \rightarrow 0$. In this regime, the benefit of explicit regularization vanishes and the min-norm interpolator indeed achieves the optimal \textit{statistical} risk. 
% Observe that the optimal statistical parameter is given in closed-form as $\lambdaS = \SNRstat^{-1}\gamma$ which is indeed $0$ when $n \gg d$. 
The optimal causal regularization in this regime is given explicitly by $\lambdaC =\genConf/(1 - \genConf)$ for $0 \leq \genConf \leq 1$ and $\lambdaC=\infty$ for $\genConf>1$.
% \begin{equation*}
%     \lambdaC = \begin{cases}
%             \frac{\genConf}{1 - \genConf}, &0 \leq \genConf \leq 1  \\
%             \infty, & \genConf > 1
%     \end{cases}\,,
% \end{equation*}
This is strictly positive and increasing in the confounding strength $\genConf$, and in fact diverges as $\genConf$ approaches $1$ (see Theorem~\ref{thm:conf_incr_reg}). 
\section{On Optimal Regularization}\label{sec:optimal_regularization}
% \TODO{Somewhere: argue that causal bias decreases further with $\lambda$, which is why causal regularization can be stronger than statistical}
% \begin{figure}
%     \centering
%     \includegraphics[width=.5\linewidth]{COLT22/fig/oracle_lam_grid_generalized.pdf}
%     \caption{
%     Each plots shows the optimal causal regularization $\lambdaC$ and optimal statistical regularization $\lambdaS$ against the overparameterization ratio $\gamma$. All plots in the same row have the same observational distribution with statistical signal-to-noise ratio increasing from bottom to top. The underlying causal models have different parameters such that their confounding strength increases from left to right. The optimal causal regularization is always larger than the optimal statistical regularization. This difference grows as the confounding strength increases and as the statistical signal-to-noise ratio decreases.
%     }
%     \label{fig:oracle_lam_grid}
% \end{figure}
% %
% \begin{figure}
%     \centering
%     \includegraphics[width=.5\linewidth]{COLT22/fig/oracle_risk_grid_generalized.pdf}
%     \caption{
%     The risk curves corresponding to Figure~\ref{fig:oracle_lam_grid}.
%     }
%     \label{fig:oracle_risk_grid}
% \end{figure}
% The choice of the regularization parameter is clearly crucially important in learning good predictive models both for statistical as well for causal considerations. 
In this section, we investigate two key questions which are natural in the context of our work. How does the optimal causal regularization $\lambdaC$ compare to the optimal statistical regularization $\lambdaS$? What is the dependence of the optimal causal regularization $\lambdaC$ on the confounding strength $\genConf$?
\paragraph{Statistical vs.\ causal optimal regularization} 
% When the training and test distributions coincide, that is, for statistical generalization, approaches such as cross-validation or information criterion (for example, AIC or BIC) can be utilized to estimate the regularization parameter for optimal out-of-sample generalization. 
When the training and test distributions coincide, approaches such as cross-validation or information criteria (for example AIC or BIC) can be used to estimate the regularization parameter for optimal out-of-sample generalization. 
However, choosing the correct regularization parameter for causal learning can be challenging since we do not observe any data from the interventional distribution. 
% To elicit insights into the optimal choice of $\lambda$ for causal questions, a natural approach would be to compare it to the optimal choice of $\lambda$ in the statistical setting which can usually be estimated from data. 
To understand the optimal causal regularization, it is natural to compare it to the optimal statistical regularization, which can usually be estimated from data. Interestingly, our analysis reveals that when confounding strength is positive $\genConf > 0$ and in particular under the ICM one needs to regularize more strongly for causal generalization than for statistical generalization. This resolves a recent conjecture in \citet{Jan:2019} which suggests that one may generally need to regularize more strongly for causal learning than for statistical learning. However, when the confounding strength is negative, that is, when the causal signal dominates the statistical signal, the optimal causal regularization $\lambdaC$ can actually be smaller than the optimal statistical regularization $\lambdaS$ and as we saw earlier in Section \ref{sec:benign}, it can even be negative. We formally present this result in Theorem \ref{thm:strong_reg_caus_vs_stat_gen}. 
% \TODO{and a visualization of these findings can be seen in Figure X.}
\begin{restatable}[Optimal Statistical vs.\ Causal Regularization]{theorem}{GenStrongRegCausVsStat}
\label{thm:strong_reg_caus_vs_stat_gen}
For any causal model parameterized as in \eqref{eq:causal_model}, the condition $\genConf = 0$ defines a phase transition for the optimal regularization via
 \begin{equation*}
     \genConf < 0 \iff \lambdaC < \lambdaS, 
     \quad \quad  
     \genConf =0 \iff \lambdaC = \lambdaS, 
     \quad \textrm{and}\quad  
     \genConf > 0 \iff \lambdaC > \lambdaS.
 \end{equation*}
In particular under the ICM, the optimal causal regularization $\lambdaC$ is always strictly larger than the optimal statistical regularization $\lambdaS$, unless $\genConf = 0$, in which case they coincide. 
\end{restatable}
% We can also understand this result by considering the limiting expressions for bias and variance in equations \ref{eq:limBiasCridge} and \ref{eq:limVarianceCridge} and comparing them to their statistical counterparts 
% \begin{restatable}[]{theorem}{StrongRegCausVsStat}
% \label{thm:strong_reg_caus_vs_stat}
%  For any causal model defined via the structural equations in (\ref{eq:causal_model}), under the assumption that $\eta = 0$, the regularization parameter that achieves the smallest causal risk is at least as large as the regularization parameter that achieves the best statistical risk, that is, 
% %  \begin{equation*}
% %      \lambda^*_C = \argmin \limits_{\lambda \in (0, \infty)} \riskCRidge \geq  \lambda^*_S = \argmin \limits_{\lambda \in (0, \infty)} \riskSRidge.
% %  \end{equation*}
% \begin{equation*}
%     \lambdaC\geq \lambdaS
% \end{equation*}
% \end{restatable}
% \begin{proof}\textbf{(sketch).}
% % \leena{Say that they have to coincide when $\structConf$ is zero.}
% % \leena{Mention that the non-asymptotic results are in excellent agreement with the limiting expressions}
% % \TODO{Can we show that this is the case for a larger range of $\eta$? This is related to the question of whether we can show theoretically, the transition point where $\lambda^*_C  = \lambda^*_S$ for a fixed $\gamma$. So take the derivative and substitute $\lambda^*_S$ and find condition under which this is zero. Then we can generalize this result}.
% \end{proof}

\paragraph{Dependence on confounding strength $\genConf$} The problem of causal learning from observational data is one of learning under distribution shift where the distribution of the training data is shifted from that of the test distribution. As discussed earlier in Proposition~\ref{prop:conf_strength_ordering}, the confounding strength measure $\genConf$ quantifies the strength of this distribution shift. 
% Therefore, it is reasonable to expect that the optimal strength of regularization would need to increase with increasing strength of confounding. 
Therefore, we expect that the optimal regularization needs to increase with confounding strength.
% Our result in Theorem \ref{thm:conf_incr_reg} indeed confirms this phenomenon which can also be visualized in Figure \TODO{Add the regularization figure here}.
Theorem~\ref{thm:conf_incr_reg} indeed confirms this intuition. 
% \TODO{Visualized in Figure X.}

\begin{restatable}[Increasing Confounding Strength Requires Stronger Regularization]{theorem}{ConfIncrReg}
\label{thm:conf_incr_reg}
Consider the family of causal models parameterized as in (\ref{eq:causal_model}) that entail the same observational distribution. 
After fixing $\gamma \in (0, \infty)$ the optimal causal regularization $\lambdaC$ only depends on the confounding strength $\genConf$ and $\lambdaC$ is an increasing function in $\genConf$. More specifically, we can distinguish the following regimes using $\rho(\gamma,\SNRstat)=-\SNRstat^{-1}\gamma \max{\{1,\gamma\}}/(1 - \gamma)^2$:
\begin{align*}
     \genConf\leq \rho(\gamma,\SNRstat) &\implies \lambdaC=0\,,\\
   \rho(\gamma,\SNRstat)<\genConf<1&\implies\lambdaC\in(0,\infty)\quad\textrm{with}\quad \partial_{\genConf}\lambdaC(\gamma) > 0\,, \\ 
 \textrm{and} \quad \genConf \geq 1 &\implies \lambdaC=\infty.
\end{align*}
% \begin{align*}
%      \genConf\leq \rho(\gamma,\SNRstat) &\implies \lambdaC=0, \quad 1\leq\genConf \implies \lambdaC=\infty. \\
%   \textrm{and} \quad  \rho(\gamma,\SNRstat)<\genConf<1&\implies\lambdaC\in(0,\infty); \; \partial_{\zeta}\lambdaC(\gamma) > 0.
% \end{align*}
% \begin{itemize}
%     \item $\genConf\leq \rho(\gamma,\SNRstat)\implies \lambdaC=0$ \quad and\quad $1\leq\genConf \implies \lambdaC=\infty$
%     \item $\rho(\gamma,\SNRstat)<\genConf<1\implies\lambdaC\in(0,\infty)$. Here, $\lambdaC$ is strictly increasing in $\genConf$.
% \end{itemize}
\end{restatable}
\section{Discussion}
Causal learning from observational data is an extremely challenging problem because of the non-identifiability induced by hidden confounding. Typical approaches to to dealing with this non-identifiability often rely on additional information, for example observing exogenous \citep{rothenhausler2021anchor} or instrumental variables \citep{angrist1991does}, or make additional assumptions, for example no hidden confounding. When no additional information is provided, other approaches instead make certain assumptions on the underlying model. Our work is more aligned with approaches of the latter kind.
% When no additional information is available, fitting to the data or interpolation is a natural baseline.
% Typical approaches to dealing with non-identifiability rely on having access to information other than observations of the covariates and the target variable. 

Our results demonstrate that our measure of confounding strength determines the sign and the strength of regularization for optimal causal generalization. Therefore estimation of confounding strength is a crucially important problem. Under the ICM assumption, \citet{Jan:2017, Jan:2018} provide a method of estimating the confounding strength of the underlying causal model assuming a linear model in high dimensions. Given a measure of confounding strength, one can then directly optimize the causal risk to obtain the regularization parameter that is optimal for causal generalization. In this work, we focus primarily on an exhaustive treatment of causal generalization under our model. Investigating approaches for estimating confounding strength is beyond the scope of the current work.

One could further consider generalizing the assumptions we make in the paper: arbitrary covariances, shifts in the marginal distributions of covariates under interventions, more complex hypothesis classes or non-linear causal relationships. Since our simple linear model already exhibits rich behavior, we focus in this paper on thoroughly understanding the simple setting and leave such extensions for future work.

% As we show in Proposition \ref{prop:conf_strength_ordering}, under the ICM assumption, the measure of confounding strength uniquely identifies the causal model from the statistical model.

% \subfile{theory_files/general_confounding_measure.tex}
% \subfile{theory_files/ordering_causal_models.tex}
% \subfile{theory_files/generalization_gap.tex}
% \subfile{theory_files/optimal_stat_vs_caus_regularization.tex}
% \subfile{theory_files/optimal_regularization_negative.tex}

% Acknowledgments---Will not appear in anonymized version
\acks{This work has been supported by the German Federal Ministry of Education and Research (BMBF): Tübingen AI Center, FKZ: 01IS18039A, the German Research
Foundation through the Cluster of Excellence
“Machine Learning – New Perspectives for Science" (EXC 2064/1 number 390727645), and the Baden-W{\"u}rttemberg Stiftung (Eliteprogram for Postdocs project ``Clustering large evolving networks'').
The authors thank the International Max Planck Research School for Intelligent Systems (IMPRS-IS) for supporting Leena Chennuru Vankadara and Luca Rendsburg.}
\bibliography{bibliography}
\appendix

\section{Proof of Proposition~\ref{prop:risk_formulas}}\label{app:risk_formulas}
For the statistical risk, we first need one standard result about the distribution of a multivariate normal random variable conditioned on an affine function:
\begin{lemma}\label{lem:normal_conditional}
Consider a multivariate normal random variable $X\sim\Gauss{\mu}{\Sigma}$ with mean $\mu\in\R^d$ and covariance $\Sigma\in\R^{d\times d}$. Then for any $A\in\R^{k\times d}$, $b\in\R^k$, and $y\in\R^k$ it holds
\begin{align*}
    X|(AX+b)=y\sim\Gauss{\mu+\Sigma A^T(A\Sigma A^T)^+(y-A\mu-b)}{\Sigma-\Sigma A^T(A\Sigma A^T)^+A\Sigma}\,.
\end{align*}
In particular, if $X$ is a standard normal random variable  ($\Sigma=I_d$, $\mu=0$) and $b=0$, it is
\begin{align*}
    X|AX=y\sim\Gauss{A^T(AA^T)^+y}{I_d-A^T(AA^T)^+A}
\end{align*}
\end{lemma}
\begin{proof}
Let $Y=AX+b$. The joint distribution of $X$ and $Y$ is again a multivariate normal, because it can be written as an affine transformation of $X$:
\begin{align*}
    \begin{pmatrix}X \\ Y\end{pmatrix} = 
    \underbrace{\begin{pmatrix}I_d \\ A\end{pmatrix}}_{\eqqcolon A^\prime\in\R^{(d+k)\times d}} X + \underbrace{\begin{pmatrix}0_d \\ b\end{pmatrix}}_{\eqqcolon b^\prime\in\R^{d+k}} 
    = A^\prime X + b^\prime\,,
\end{align*}
which implies that
\begin{align*}
    \begin{pmatrix}X \\ Y\end{pmatrix} = A^\prime X + b^\prime
    \sim \Gauss{A^\prime \mu+b^\prime}{A^\prime \Sigma (A^\prime)^T}
    = \Gauss{
    \begin{pmatrix}
    \mu \\ A\mu +b
    \end{pmatrix}
    }{
    \begin{pmatrix}
    \Sigma & \Sigma A^T \\
    A \Sigma & A \Sigma A^T
    \end{pmatrix}
    }\,.
\end{align*}
The claim then follows from the standard formula for conditionals of multivariate normal distributions, which states that if $\begin{pmatrix}
Z_1 \\ Z_2
\end{pmatrix}\sim\Gauss{\begin{pmatrix}\mu_1 \\ \mu_2\end{pmatrix}}{\begin{pmatrix}
\Sigma_{1,1} & \Sigma_{1,2} \\
\Sigma_{2,1} & \Sigma_{2,2}
\end{pmatrix}}$, then
\begin{align*}
    Z_1|Z_2=z \sim\Gauss{\mu_1+\Sigma_{1,2}\Sigma_{2,2}^+(z-\mu_2)}{\Sigma_{1,1}-\Sigma_{1,2}\Sigma_{2,2}^+\Sigma_{2,1}}\,.
\end{align*}
\end{proof}

\RiskFormulas*
\begin{proof}
The key step for this proof is to characterize the distribution of $y$ under the $do$-intervention $y|do(x)$ and the usual observational conditional $y|x$.
We start with the proof for the causal risk under the $do$-intervention.
Intervening on $x$ under the causal model given by Eq.~\eqref{eq:causal_model} corresponds to removing all arrows to $x$, which corresponds to the structural equations
\begin{align*}
    z\sim\Gauss{0}{I_l}\,, \quad
    \varepsilon &\sim \Gauss{0}{\sigmacaus}\,,\quad
    y= x^T\beta + z^T\alpha + \varepsilon\,.
\end{align*}
In this model, $z$ acts as additional independent noise on $y$ through $z^T\alpha\sim\Gauss{0}{\lnorm{\alpha}^2}$, which implies that $y|do(x)\sim\Gauss{x^T\beta}{\lnorm{\alpha}^2+\sigmacaus}$. Equivalently, $y|do(x)$ has the same distribution as $x^T\beta + \varepsilon^\prime$ with $\varepsilon^\prime\sim\Gauss{0}{\sigmastat+\omega^2}$ because $\lnorm{\alpha}^2+\sigmacaus=\sigmastat+\omega^2$. This lets us compute the causal risk of a linear predictor $\betaGen\in\R^d$ as
\begin{align*}
    \riskCtotal(\betaGen) 
    &= \expec_{x}\expec_{y_0|do(x)}\left(x^T\betaGen-y\right)^2\\
    &=\expec_{x}\expec_{\varepsilon^\prime}\left(x^T\left(\betaGen-\beta\right)-\varepsilon^\prime\right)^2\\
    &=\expec_{x}\left(x^T\left(\betaGen-\beta\right)\right)^2-2\expec_{x}\bigg[x^T\left(\betaGen-\beta\right)\underbrace{\expec_{\varepsilon^\prime}\varepsilon^\prime}_{=0}\bigg]+\expec_{x}\expec_{\varepsilon^\prime}\left(\varepsilon^\prime\right)^2\\
    &=\lnorm{\betaGen-\beta}_\Sigma^2 + \sigmastat+\omega^2\,, \tag{$\expec_x xx^T=\Sigma$}
\end{align*}
% \begin{align*}
%     \riskCtotal(\betaGen) 
%     &= \expec_{x_0}\left[\expec_{y_0|do(x=x_0)}\left[\left(x_0^T\betaGen-y_0\right)^2\right]\right]\\
%     &=\expec_{x_0}\left[\expec_{\varepsilon^\prime}\left[\left(x_0^T\left(\betaGen-\beta\right)-\varepsilon^\prime\right)^2\right]\right]\\
%     &=\expec_{x_0}\left[\left(x_0^T\left(\betaGen-\beta\right)\right)^2\right]-2\expec_{x_0}\bigg[x_0^T\left(\betaGen-\beta\right)\underbrace{\expec_{\varepsilon^\prime}\varepsilon^\prime}_{=0}\bigg]+\expec_{x_0}\expec_{\varepsilon^\prime}\left(\varepsilon^\prime\right)^2\\
%     &=\lnorm{\betaGen-\beta}_\Sigma^2 + \sigmastat+\omega^2\,, \tag{$\expec_{x_0}x_0x_0^T=\Sigma$}
% \end{align*}
which proves the claim for the causal risk. 
The proof for the statistical risk is analogous once we have characterized the conditional distribution $y|x$ under the causal model. Recall that $\Sigma=MM^T$, $\Gamma=\Sigma^+M\alpha$, and $\omega^2=\lnorm{\Gamma}_\Sigma^2$. We first observe that $x=Mz$ is a linear map of the Gaussian distribution $z\sim\Gauss{0}{I_l}$, for which Lemma~\ref{lem:normal_conditional} yields
\begin{align*}
    z|x&\sim\Gauss{M^T(MM^T)^+x}{I-M^T(MM^T)^+M}\\
    \text{and therefore}\quad z^T\alpha|x &\sim\Gauss{\alpha^TM^T(MM^T)x}{\lnorm{\alpha}^2-\alpha^TM^T(MM^T)^+M\alpha}\\
    &=\Gauss{x^T\Gamma}{\lnorm{\alpha}^2-\norm{\Gamma}^2_\Sigma}\,,
\end{align*}
where the last equality used the property of the pseudo-inverse
\begin{align*}
    \alpha^TM^T(MM^T)^+M\alpha
    =\alpha^TM^T\Sigma^+M\alpha
    =\alpha^TM^T\Sigma^+\Sigma\Sigma^+M\alpha
    =\Gamma^T\Sigma\Gamma
    =\lnorm{\Gamma}_\Sigma^2
    =\omega^2\,.
\end{align*}
Since $y=x^T\beta+z^T\alpha+\varepsilon$, it follows that
\begin{align*}
    y|x\sim\Gauss{x^T(\beta+\Gamma)}{\sigmacaus+\lnorm{\alpha}^2-\omega^2}
    =\Gauss{x^T\betaStat}{\sigmastat}\,,
\end{align*}
which concludes the proof.
\end{proof}

\section{Proofs for Section~\ref{subsec:main_results}}\label{app:causal_results}
The bias-variance decomposition of the causal risk is based on the following general lemma:
\begin{lemma}[Bias-Variance Decomposition for General Norm]\label{lem:general_bv_decomp}
Consider a random variable $Z$ on $\R^d$, a constant $c\in\R^d$, and the general norm $\lnorm{x}^2_A= x^TA x$ for some positive-definite $A\in\R^{d\times d}$. Then we have the decomposition
\begin{align*}
    \expec_Z\lnorm{Z-c}^2_A = \lnorm{\expec Z-c}^2_A + \expec_Z\lnorm{Z-\expec_ZZ}^2_A\,.
\end{align*}
An alternative form of the variance term is given by $\expec_Z\lnorm{Z-\expec_ZZ}^2_A=\Tr\left[\Cov Z\cdot A\right]$.
\end{lemma}
\begin{proof}
Let $\expec\coloneqq\expec_Z$ and $\mu\coloneqq\expec Z$. It is
\begin{align*}
    \expec\lnorm{Z-c}_A^2 
    &=\expec\lnorm{(Z-\mu)+(\mu-c)}_A^2\\
    &=\expec\lnorm{Z-\mu}_A^2 +\expec\lnorm{\mu-c}_A^2 +2\underbrace{\expec(Z-\mu)^T}_{=0}A(\mu-c)\\
    &=\expec\lnorm{Z-\mu}_A^2 +\expec\lnorm{\mu-c}_A^2\,,
\end{align*}
which proves the first part of the statement. For the second part, let $\Sigma_Z\coloneqq\expec ZZ^T$ and denote the Hadamard product between matrices $A,B\in\R^{d\times d}$ by $(A\odot B)_{i,j}=A_{i,j}B_{i,j}$. It is
\begin{align*}
\expec\lnorm{Z-\mu}_A^2
&=\expec Z^TAZ - 2\expec Z^T A \mu + \mu^T A \mu\\
&=\sum_{i,j=1}^n (\Sigma_Z\odot A)_{i,j}-\mu^T A \mu\\
&=\Tr\left[\Sigma_Z\cdot A\right] - \mu^TA\mu\tag{$\sum_{i,j=1}^n(A\odot B)_{i,j}=\Tr(A\cdot B)$}\\
&=\Tr\left[\Sigma_Z\cdot A\right]- \Tr\left[A\mu\mu^T\right]\tag{$\Tr(ba^T)=a^Tb$}\\
&=\Tr\left[(\Sigma_Z-\mu\mu^T)\cdot A\right]\tag{$\Tr(B)=\Tr(B^T)$ and linearity of trace}\\
&=\Tr\left[\Cov Z\cdot A\right] \tag{$\Cov Z = \expec ZZ^T-\mu\mu^T$}\,.
\end{align*}
\end{proof}

%%%% Proposition: causal bias-variance decomposition
\begin{proposition}[Causal Bias-Variance Decomposition for the Ridge Estimator]\label{prop:causal_bv_decomp}
For any $\lambda >0$, the expectation over the causal risk of the ridge regression estimator $\betaRidge$ conditioned on $X$ admits the bias-variance decomposition
\begin{align}\label{eq:causal_bias_variance}
    \riskCRidge= \underbrace{\norm{\expec_{Y|X}\betaRidge-\beta}^2_{\Sigma}}_{\eqqcolon\biasCRidge}  + \underbrace{ \expec_{Y \vert X} \norm{\betaRidge - \expec_{Y \vert X} \betaRidge}_{\Sigma}^2}_{\eqqcolon\varianceCRidge}+\Epssigmastat+\norm{\Gamma}^2_\Sigma\,,
\end{align}
where $ \biasCRidge 
    =\norm{(I - (\empcov + \lambda I_d)\empcov)\betaStat -\Gamma}^2_{\Sigma}$ and $\varianceCRidge= \frac{\sigmastat}{n} \Tr[\empcov(\empcov+\lambda I_d)^{-2} \Sigma]$. The empirical covariance matrix of $X$ is denoted by $\empcov\coloneqq X^TX/n$.
\end{proposition}
\begin{proof}
Recall that $\riskCRidge=\expec_{Y|X}\lnorm{\betaRidge-\beta}_\Sigma^2$.
The first part of the statement follows directly from Lemma~\ref{lem:general_bv_decomp} with $\betaRidge$ as a random variable in $Y|X$ and $\beta$.
The remainder of the proof consists of computing expectation and covariance of the ridge regression solution $\betaRidge=\betaRidge(X,Y)$ under the distribution $Y|X$.
The samples $(X, Y)$ are drawn from the observational distribution of the causal model defined in Eq.~\eqref{eq:causal_model}. As shown in the proof of Proposition~\ref{prop:risk_formulas}, the corresponding conditional distribution is $y|x\sim\Gauss{x^T\betaStat}{\sigmastat}$. Since $(X,Y)$ consist of independent draws, this implies $Y|X\sim\Gauss{X\betaStat}{\sigmastat I_n}$. Together with $\betaRidge=(X^TX+n\lambda I)^{-1}X^TY$ this yields
\begin{align*}
    \betaRidge|X&\sim\Gauss{(X^TX+n\lambda I)^{-1}X^TX\betaStat}{(X^TX+n\lambda I)^{-1}X^T\sigmastat I_n X(X^TX+n\lambda I)^{-1}}\\
    &=\Gauss{\left(\empcov+\lambda I_d\right)^{-1}\empcov \betaStat}{\frac{\sigmastat}{n}\left(\empcov+\lambda I_d\right)^{-1}\empcov \left(\empcov+\lambda I_d\right)^{-1}}\,.
\end{align*}
The characterizations of $\biasCRidge$ and $\varianceCRidge$ then simply follow from plugging in expectation and covariance of $\betaRidge$:
\begin{align*}
    \biasCRidge 
    =\lnorm{\expec_{Y|X}\betaRidge-\beta}^2_{\Sigma}
    =\lnorm{\left(\empcov+\lambda I_d\right)^{-1}\empcov\betaStat-\beta}^2_{\Sigma}
    &=\lnorm{\left(I-\projectRidge\right)\left(\beta+\Gamma\right)-\beta}^2_{\Sigma}\\
    &=\lnorm{\projectRidge\beta-(I-\projectRidge)\Gamma}^2_{\Sigma}
\end{align*} 
and, using the alternate form of the variance term from Lemma~\ref{lem:general_bv_decomp},
\begin{align*}
    \varianceCRidge 
    =\Tr\left[\Cov_{Y|X}\betaRidge\cdot \Sigma\right]
    &=\Tr\left[\frac{\sigmastat}{n}\left(\empcov+\lambda I_d\right)^{-1}\empcov \left(\empcov+\lambda I_d\right)^{-1}\cdot \Sigma\right]\\
    &= \frac{\sigmastat}{n} \Tr\left[\empcov\left(\empcov+\lambda I_d\right)^{-2} \Sigma\right]\,,
\end{align*}
where the last equality used that $\left(\empcov+\lambda I_d\right)^{-1}$ commutes with $\empcov$.
\end{proof}

% \CausalRidge*
\begin{mainExtended}
Let $\norm{\beta}^2=\signalcaus$, $\norm{\Gamma}^2=\omega^2$, $\scalprod{\Gamma}{\beta} = \eta$, and $\sigma_{\tilde{\epsilon}}^2 = \sigmastat$. Then as $n,d\to\infty$ such that $d/n\to\gamma\in (0,\infty)$, it holds almost surely  in $X$ for every $\lambda>0$ that
\begin{align*}
    \biasCRidge&\to
    \limBiasArgC{\lambda}\coloneqq\omega^2 +  \signalstat \lambda^2 m'(-\lambda)-  2 ( \omega^2 + \eta) \lambda m(-\lambda)\quad\text{and}\tag{\ref*{eq:limBiasCridge}}\\
     \varianceCRidge&\to
     \limVarianceArgC{\lambda}\coloneqq\sigmastat\gamma(m(-\lambda) - \lambda m'(-\lambda))\,,\tag{\ref*{eq:limVarianceCridge}}
\end{align*}
where 
% $m(\lambda) = \frac{(1 - \gamma - \lambda) -  \sqrt{(1  -  \gamma  -  \lambda)^2 - 4 \gamma \lambda}}{2 \gamma \lambda}$ 
$m(\lambda) = ((1 - \gamma - \lambda) -  \sqrt{(1  -  \gamma  -  \lambda)^2 - 4 \gamma \lambda}) / (2 \gamma \lambda)$ 
and $\signalstat = r^2 + \omega^2 + 2 \eta$.
Therefore $\riskCRidge\to\limRiskArgC{\lambda}\coloneqq\limBiasArgC{\lambda}+\limVarianceArgC{\lambda}+\sigmastat+\omega^2$. 
The corresponding limiting quantities for the min-norm interpolator can be obtained by taking the limit $\lambda\to 0^+$ in equations \eqref{eq:limBiasCridge} and \eqref{eq:limVarianceCridge}, which yields
\begin{align*}
    \biasCMinNorm\to
    \limBiasArgC{0}=
    \begin{cases}
    \omega^2, &\gamma<1\\
    \omega^2+ (\signalcaus-\omega^2)(1-\frac{1}{\gamma}), &\gamma>1
    \end{cases}
    \,,\quad
    \varianceCMinNorm \to
    \limVarianceArgC{0}=
    \begin{cases}
    \sigmastat\frac{\gamma}{1-\gamma}, &\gamma<1\\
    \sigmastat\frac{1}{\gamma-1}, &\gamma>1
    \end{cases}\,.
\end{align*}
Therefore $\riskCMinNorm\to\limRiskArgC{0}=\limBiasArgC{0}+\limVarianceArgC{0}+\sigmastat+\omega^2$.
\end{mainExtended}

\begin{proof}
From Proposition~\ref{prop:causal_bv_decomp}, the causal risk $\riskCRidge$ can be decomposed as a sum of the causal bias $\biasCRidge$, and causal variance $\varianceCRidge$. In what follows, we derive the limiting expressions for $\biasCRidge$ and $\varianceCRidge$ to obtain the limiting causal risk for any $\gamma \in (0, \infty).$

\paragraph{Limiting expressions for causal bias}
\begin{align*}
    \biasCRidge &= \norm{\beta - \expec_{\vert X} \betaRidge}^2_{\Sigma}  = \norm{\projectRidge\beta-(I-\projectRidge)\Gamma}^2 && (\Sigma = I) \\
    &= \lnorm{\projectRidge (\beta + \Gamma) - \Gamma}^2 \\
    &= \norm{\projectRidge \betaStat}^2 + \norm{\Gamma}^2 - 2 \scalprod{\Gamma}{\projectRidge(\betaStat)} 
\end{align*}

First, let us consider the sequence of functions given by
\begin{align*}
   \norm{\projectRidge \betaStat}^2  &=\norm{(I - (\empcov + \lambda I)^{-1} \empcov) \betaStat}^2 \\ &= \lnorm{\lambda ((\empcov + \lambda I)^{-1})\betaStat}^2 && \textrm{(Add and subtract $\lambda I$)} \\
    &= \lambda^2 \betaStat^T (\empcov + \lambda I)^{-2} \betaStat^T \\
    &= \lambda^2 \Tr \left [ \betaStat\betaStat^T  (\empcov + \lambda I)^{-2} \right ]
\end{align*}

To derive the limiting expression for this sequence, we utilize the ``derivative trick''. This technique has been employed in a similar context in \citet{Dob:2018}. More generally similar terms (although not identical) often also arise in the analysis of the statistical of the ridge regression estimator and therefore one can find similar approaches to deriving the limiting expressions for such terms in the statistical analysis for ridge regression (for example, \citet{Has:2019, Dob:2018, dicker2016ridge}). Here, we include a self-contained proof of the result. 

The idea relies on an application of Vitali's convergence theorem (see \citet[Lemma 2.14]{bai2010spectral}) to obtain the limit of derivatives of a sequence of functions analytic on some domain $D \subset \mathbb{C}$ by the derivative of the limit of the sequence of functions. Observe that  
\begin{equation*}
    \Tr \left [ (\beta + \Gamma)(\beta + \Gamma)^T  (\empcov + \lambda I)^{-2} \right ] = \frac{\partial  }{\partial \lambda}-\Tr \left [ (\beta + \Gamma)(\beta + \Gamma)^T  (\empcov + \lambda I)^{-1} \right ]
\end{equation*}

By recognizing the quantity $(\empcov + \lambda I)^{-1}$ as the resolvent $Q(-\lambda)$, we can invoke the Marchenko-Pastur Theorem due to \citet{marvcenko1967distribution, silverstein1995strong} which states that the Stieltjes transform of the empirical distribution $\hat{m(z)}$ of eigenvalues of $\empcov$ converges almost surely to the Stieltjes transform $m(z)$ of the empirical spectral distribution given by the Marchenko-Pastur Law $F$ for any $z \in \mathbb{C}/\R^+$. \footnote{While the convergence result in \citet{silverstein1995strong} is stated for $z \in \mathbb{C}^+ = \mycurls{z = u + iv \in \mathbb{C} \vert Im(z) = v > 0}$, it can be extended to $z \in \mathbb{C} / \R^+$ following standard arguments for convergence of sequences of analytic functions (see \citet[Proposition 2.2]{hachem2007deterministic}) via Vitali's convergence theorem or Montel's theorem. See \citet[Proof of Theorem 1, Page 14]{Rub:2011} for an example of this argument.} That is, we have for all $\lambda > 0$, 
\begin{equation*}
   \frac{1}{d} \Tr \left [ (\empcov + \lambda I)^{-1} \right] \xrightarrow{a.s} m_F(-\lambda)
\end{equation*}
\citet[Theorem 1]{Rub:2011} provide a generalization of this result which includes providing almost sure convergence of quadratic forms of resolvents of the form $u^T (\empcov - z I) v$ for sequences of vectors $\mycurls{u}, \mycurls{v}$ such that their outer product $u v^T$ has a bounded trace norm for any $z \in \mathbb{C} / \R^+$. By this result, it is easy to verify that for any $\lambda > 0$,
\begin{equation*}
    \Tr \left [ \betaStat \betaStat^T  (\empcov + \lambda I)^{-1}\right ] \xrightarrow{a.s} m_F(-\lambda)\signalstat
\end{equation*}
It is easy to see that the sequence of functions $\mycurls{f_d(\lambda) = \Tr \left [ \betaStat \betaStat^T (\empcov + \lambda I)^{-1} \right ]}$ is analytic for $\lambda > 0$. Furthermore, for any $\lambda > 0$, the absolute value of the sequence of functions $\mycurls{f_d(\lambda)}$ is uniformly bounded in $d$ since
\begin{equation*}
    \vert f_d(\lambda)\vert \leq \Tr [\betaStat \betaStat^T] \frac{1}{\lambda} \leq  \frac{\signalstat}{\lambda}
\end{equation*}
Therefore, by Vitali's convergence theorem, it holds (almost surely) that for every $\lambda > 0$, the derivatives of the sequence of functions $f_1, f_2, \cdots$ converges to the derivative of their limit and we have
\begin{equation*}
   \lambda^2 \Tr \left [ \betaStat\betaStat^T  (\empcov + \lambda I)^{-2} \right ] \rightarrow  \lambda^2 \signalstat m'_F(-\lambda),
\end{equation*}
where $m'_F(-\lambda)$ denotes the derivative of the Stieltjes transform of the Marchenko-Pastur Law evaluated at $- \lambda$. 

To obtain the limiting function of the sequence $\scalprod{\Gamma}{\projectRidge \betaStat}$, observe that %
\begin{align*}
    \scalprod{\Gamma}{\projectRidge \betaStat} & = \lambda \scalprod{\Gamma}{(\empcov + \lambda I)^{-1}\betaStat} = \lambda \Tr[\betaStat \Gamma^T (\empcov + \lambda I)^{-1}] \xrightarrow{a.s} \lambda (\omega^2 + \eta) m_F(-\lambda),
\end{align*}
where the limit is obtained by invoking \citet[Theorem 1]{Rub:2011}.

Therefore, we have that as $n, d \rightarrow \infty$ and $d/n \rightarrow \gamma$, 
\begin{equation*}
    \biasCRidge \xrightarrow{a.s} \omega^2 + \signalstat \lambda^2 m_F'(-\lambda) - 2(\omega^2 + \eta) \lambda m_F(-\lambda).
\end{equation*}
\textbf{Limiting expressions for causal variance.}

By recalling the expression for variance we have 
\begin{align*}
    \varianceCRidge &= \frac{\sigmastat}{n} \Tr \left [ \empcov(\empcov + \lambda I)^{-2} \right] \\
    &= \frac{\sigmastat}{n} \Tr \left [(\empcov + \lambda I - \lambda I)(\empcov + \lambda I)^{-2} \right ] \\
    &= \sigmastat \frac{d}{n} \Tr \left [\frac{1}{d}(\empcov + \lambda I )^{-1} - \frac{1}{d} \lambda(\empcov + \lambda I)^{-2} \right ] \\
\end{align*}
By Marchenko-Pastur Theorem \citep{marvcenko1967distribution, silverstein1995strong}, we already know that for any $\lambda >0$
\begin{align*}
     \Tr \left [\frac{1}{d} (\empcov + \lambda I)^{-1} \right ] \rightarrow m_F(-\lambda)
\end{align*}
Further, recognizing that 
\begin{equation*}
  -  \Tr \left [\frac{1}{d} (\empcov + \lambda I)^{-2} \right ] = \frac{\partial}{\partial \lambda} \Tr \left [\frac{1}{d} (\empcov + \lambda I)^{-1} \right ]
\end{equation*}
and that $\vert \Tr  [\frac{1}{d} (\empcov + \lambda I)^{-1}  ] \vert \leq \frac{1}{\lambda}$, we can again invoke Vitali's convergence theorem to obtain the limit of the derivatives by taking the derivative of the limit to obtain
\begin{align*}
    \varianceCRidge = \sigmastat \gamma (m_F(-\lambda) - \lambda m_F'(-\lambda)).
\end{align*}
%
% From \citet[Theorem 1]{Rub:2011}, we have that for any $\lambda > 0$ and for any sequence of nonrandom matrices $\mycurls{A_n} \in \R^{d \times d}$ with a uniformly bounded trace norm
%
% \begin{equation*}
%     \Tr \big [ A_n ( (\empcov + \lambda I)^{-1} - c_n(\lambda) I)  \big ] \rightarrow 0 \; \; \textrm{almost surely} \; \textrm{as} \; n,d \rightarrow \infty, n/d \rightarrow \gamma
% \end{equation*}
% for a deterministic sequence $c_n(\lambda)$. By the Marchenko-Pastur Theorem \TODO{Add ref}, we know that the Stieltjes transform of the spectral distribution of the empirical covariance matrix $\empcov$ given by $\hat{m}(\lambda) = \frac{1}{d} \Tr \big [ (\empcov - z I)^{-1} \big ]$ converges almost surely(a.s) to the Stieltjes transform $m(z)$ of the
%
Marchenko-Pastur Law admits an explicit form under our model assumptions (see for example, \citep[Page 52]{bai2010spectral}) for any $z \in \mathbb{C}^{+}$ (which can be extended by analytic continuity arguments for any $z \in \mathbb{C} / \R^{+}$) and is given by
\begin{equation*}
    m_F(z) = \frac{1 - \gamma - z - \sqrt{(1 - \gamma - z)^2 - 4 \gamma z}}{2 \gamma z}.
\end{equation*}

Following arguments similar to \citet{Dob:2018, Has:2019} for exchanging the limits $n, d \rightarrow \infty$ and $\lambda \rightarrow 0^+$, we can derive the limiting expressions for the causal bias and variance of the min-norm estimator.

% . Therefore, $c_n(z) \rightarrow m(-z)$ and we have 

% \begin{equation*}
%     -\Tr \left [ (\beta + \Gamma)(\beta + \Gamma)^T  (\empcov + \lambda I)^{-1} \right ] \xrightarrow{a.s} -m(-\lambda) \signalstat.
% \end{equation*}
\end{proof}

%%%%%%%%%%%%%%%%%%%%%%%%%%%%%%%%%%%%%%%%%%%%%%%%%%%%
\section{Asymptotics for the Statistical Risk}\label{app:statistical_asymptotics}
The following theorems describes the limiting expressions for the statistical risk analogue to the causal results from Theorem~\ref{thm:causal_ridge}.

\begin{theorem}[\textbf{Limiting Statistical Bias-Variance Decompositions}]\label{thm:statistical_results}
Let $\betaMinNorm$ be the min-norm interpolator.
Then as $n,d\to\infty$ such that $d/n\to\gamma\in (0,\infty)$, it holds almost surely  in $X$ that
\begin{align}
    \biasSMinNorm\to
    \limBiasArgS{0}=
    \begin{cases}
    0, &\gamma<1\\
    \signalstat(1-\frac{1}{\gamma}), &\gamma>1
    \end{cases} 
    \,,\quad
    \varianceSMinNorm \to
    \limVarianceArgS{0}=
    \begin{cases}
    \sigmastat\frac{\gamma}{1-\gamma}, &\gamma<1\\
    \sigmastat\frac{1}{\gamma-1}, &\gamma>1
    \end{cases}\,
\end{align}
and therefore, $\riskSMinNorm\to\limRiskArgS{0}=\limBiasArgS{0}+\limVarianceArgS{0}+\sigmastat$.\\
For $\lambda>0$ and the corresponding ridge regression estimator $\betaRidge$, it holds almost surely  in $X$ that
\begin{align}
    \biasSRidge\to\limBiasArgS{\lambda}= \signalstat\lambda^2 m'(-\lambda)
     \,,\quad
     \varianceSRidge\to\limVarianceArgS{\lambda}= \sigmastat\gamma(m(-\lambda) - \lambda m'(-\lambda)),
\end{align}
where $m(\lambda) = \frac{(1 - \gamma - \lambda) -  \sqrt{(1  -  \gamma  -  \lambda)^2 - 4 \gamma \lambda}}{2 \gamma \lambda}$. 
Therefore, $\riskSRidge\to\limRiskArgS{\lambda}=\limBiasArgS{\lambda}+\limVarianceArgS{\lambda}+\sigmastat$.
\end{theorem}
\begin{proof}
As stated in the main paper, this result for the statistical model was already proven in \citet{Has:2019}.
\end{proof}

%%%%%%%%%%%%%%%%%%%%%%%%%%%%%%%%%%%%%%%%%%%%%%%%%%%%
\section{Proof of Proposition~\ref{prop:conf_strength_ordering}}\label{app:conf_strength_ordering}
\orderingCausal*
\begin{proof}
For any fixed $\lambda \in (0, \infty)$, the difference in limiting causal risks incurred by $\betaRidge$ on causal models $C_1$ and $C_2$ is given by 
\begin{align*}
    \limRiskArgC{1}(\gamma, \lambda) - \limRiskArgC{2}(\gamma, \lambda) &= 2\signalstat \big ( (\frac{\omega_1^2}{\signalstat} - \frac{\omega_2^2}{\signalstat}) - (\genConf_1 - \genConf_2)\lambda m(-\lambda)\big) \\
    &=  2\signalstat \big ( (\genConf_1 - \genConf_2)(1 - \lambda m(-\lambda)) - (\eta_1 - \eta_2) \big) \\
    &= 2\signalstat \big ( (\genConf_1 - \genConf_2)(1 - \lambda m(-\lambda)) - (\eta_1 - \eta_2) \big)
\end{align*}
Since, as shown below, $(1 - \lambda m(-\lambda)) > 0$ for any $\lambda, \gamma \in (0, \infty)$, it holds that 
 \begin{equation*}
     \genConf_1 > \genConf_2, \; \; \eta_1 \leq \eta_2 \implies \limRiskArgC{1}(\gamma, \lambda) > \limRiskArgC{2}(\gamma, \lambda).
 \end{equation*}
\begin{align*}
  1 -  \lambda m(-\lambda) &= 1 - \frac{\gamma-1-\lambda+\sqrt{(1 + \lambda + \gamma)^2 - 4\gamma}}{2\gamma} \\
  &= \frac{(1 + \gamma + \lambda) - \sqrt{(1 + \lambda + \gamma)^2 - 4\gamma}}{2 \gamma} \\
  & > 0 &&(\textrm{since } \gamma > 0)
\end{align*}
\end{proof}

\section{Proofs for Sections~\ref{sec:benign}~and~\ref{sec:optimal_regularization}}
We start with a technical lemma that we need in the proofs of the following theorems. It controls a function that appears in the derivative of the limiting causal riks $\partial_\lambda\limRiskArgC{\lambda}$.
% Lemma for proof of the cases $\lambdaC\in\{0,\infty\}$
\begin{lemma}\label{lem:properties_func_f}
For $\lambda\geq 0$ and $\gamma,S>0$ consider the function 
\begin{align*}
    f(\lambda,\gamma, S)=2\gamma\frac{\lambda-S^{-1}\gamma}{(1+\lambda+\gamma-\sqrt{(1+\lambda+\gamma)^2-4\gamma})((1+\lambda+\gamma)^2-4\gamma)}\,.
\end{align*}
This function has the following properties
\begin{enumerate}[label={\upshape(\roman*)}]
    \item\label{item:f_increasing} $f$ is increasing in $\lambda$,
    \item\label{item:f_infty} $f(\lambda, \gamma, S)\xrightarrow[\lambda\to\infty]{}1$, and
    \item\label{item:f_0} $f(\lambda,\gamma, S)\xrightarrow[\lambda\to 0]{}
    \begin{cases}
    -S^{-1}\frac{\gamma}{(\gamma-1)^2}, &\gamma<1\\
    -\infty, &\gamma=1\\
    -S^{-1}\frac{\gamma^2}{(\gamma-1)^2}, &\gamma>1
    \end{cases}$\,.
\end{enumerate}
\end{lemma}
\begin{proof}
For readability, we use the shorthand notations $x=1+\lambda+\gamma$ and $\varphi=x^2-4\gamma$, under which $f$ is given by
\begin{align*}
    f(\lambda,\gamma, S)=2\gamma\frac{\lambda-S^{-1}\gamma}{(x-\sqrt{\varphi})\varphi}\,.
\end{align*}
% f increasing
\ref{item:f_increasing}
The partial derivative of $f$ in $\lambda$ is given by
\begin{align*}
    \partial_\lambda f(\lambda,\gamma, S)
    &=2\gamma\frac{(x-\sqrt{\varphi})\varphi-(\lambda-S^{-1}\gamma)\left[(1-\frac{x}{\sqrt{\varphi}})\varphi + 2x(x-\sqrt{\varphi})\right]}{(x-\sqrt{\varphi})^2\varphi^2}\\
    &=\underbrace{\frac{2\gamma}{(x-\sqrt{\varphi})\varphi^2}}_{> 0}\underbrace{\left[\varphi-(\lambda-S^{-1}\gamma)(2x-\sqrt{\varphi})\right]}_{\eqqcolon g(\lambda)}\,,
\end{align*}
where the first fraction is positive because $\varphi> x^2$ and $x-\sqrt{\varphi}>0$. It is therefore sufficient to show $g(\lambda)\geq0$ for $\partial_\lambda f(\lambda,\gamma, S)\geq0$. We first get rid of the $S$ term via
\begin{align*}
    g(\lambda)=\varphi-(\lambda-S^{-1}\gamma)\underbrace{(2x-\sqrt{\varphi})}_{\geq 0}
    \geq\varphi-\lambda(2x-\sqrt{\varphi})\,.
\end{align*}
Finally, we lower bound $\sqrt{\varphi}$ in two different ways depending on $\gamma$. For $\gamma\leq1$, it is $\varphi=(1+\lambda-\gamma)^2+4\gamma\lambda$ and therefore $\sqrt{\varphi}\geq 1+\lambda-\gamma=x-2\gamma$. This yields
\begin{align*}
    g(\lambda)
    \geq \varphi-\lambda(2x-\sqrt{\varphi})
    \geq \varphi-\lambda(x+2\gamma)
    =(1-\gamma)\lambda + (\gamma-1)^2
    \geq 0\,.
\end{align*}
For $\gamma>1$, it is $\varphi=(-1+\lambda+\gamma)^2+4\lambda$ and therefore $\sqrt{\varphi}\geq -1+\lambda+\gamma=x-2$. This yields
\begin{align*}
    g(\lambda)
    \geq \varphi-\lambda(2x-\sqrt{\varphi})
    \geq \varphi-\lambda(x+2)
    =(\gamma-1)\lambda + (\gamma-1)^2
    \geq 0\,.
\end{align*}
In summary, we have shown $\partial_\lambda f(\lambda,\gamma, S)\geq g(\lambda)\geq 0$.

% f at \infty
\ref{item:f_infty} With the first order Taylor approximation $1-\sqrt{1-h}=1/2h+\mathcal{O}(h^2)$, we get
\begin{align*}
    (x-\sqrt{\varphi})\varphi
    =\left(1-\sqrt{1-\frac{4\gamma}{x^2}}\right)x\varphi
    =\left(\frac{2\gamma}{x^2}+\mathcal{O}(\lambda^{-4})\right)x\varphi
    =2\gamma x + \mathcal{O}(\lambda^{-1})
    =2\gamma\lambda + \mathcal{O}(1)\,,
\end{align*}
which yields
\begin{align*}
    f(\lambda,\gamma, S)
    =2\gamma\frac{\lambda-S^{-1}\gamma}{(x-\sqrt{\varphi})\varphi}
    =\frac{2\gamma\lambda - 2S^{-1}\gamma^2}{2\gamma\lambda + \mathcal{O}(1)}
    \xrightarrow[\lambda\to\infty]{}1\,.
\end{align*}

% f at 0
\ref{item:f_0}
The denominator satisfies
\begin{align*}
    (x-\sqrt{\varphi})\varphi
    \xrightarrow[\lambda\to 0]{}
    (1+\gamma-\abs{\gamma-1})(\gamma-1)^2=
    \begin{cases}
    2\gamma(\gamma-1)^2, &\gamma<1\\
    0, &\gamma=1\\
    2\gamma-1)^2, &\gamma>1
    \end{cases}\,.
\end{align*}
Since $\lambda-S^{-1}\gamma\xrightarrow[\lambda\to 0]{}S^{-1}\gamma<0$, the claim follows.
\end{proof}
Recall that the optimal causal regularization is defined as the minimizer of the causal risk $\lambdaC(\gamma)=\arginf_{\lambda \in (0, \infty)} \limRiskArgC{\lambda}$. The following lemma distinguishes between three different regimes of the risk function $\limRiskArgC{\lambda}$ depending on the confounding strength $\genConf$.
\begin{lemma}[Regimes of the Optimal Causal Regularization]
\label{lem:optimal_causal_lambda_characterization}
For any causal model parameterized as in (\ref{eq:causal_model}), we can distinguish the following regimes of $\lambdaC(\gamma)$:
\begin{enumerate}
    \item The function $\lambda\mapsto\limRiskArgC{\lambda}$ is increasing (which implies $\lambdaC(\gamma)=0$), if and only if $\gamma\neq 1$ and
    % \begin{equation*}
    %     \genConf \leq  \begin{cases}
    % -\frac{\SNRstat^{-1}\gamma}{(\gamma-1)^2}, &\gamma<1\\
    % -\frac{\SNRstat^{-1}\gamma^2}{(\gamma-1)^2}, &\gamma>1
    % \end{cases}
    % \end{equation*}
    \begin{equation*}
        \genConf\leq -\SNRstat^{-1}\frac{\gamma \max{\{1,\gamma\}}}{(1 - \gamma)^2}\,.
    \end{equation*}
    \item For any $\gamma>0$, the function $\lambda\mapsto\limRiskArgC{\lambda}$ is decreasing (which implies $\lambdaC(\gamma)=\infty$) if and only if $\genConf\geq 1$.
    \item For any $\genConf \in \R$, $\gamma \in (0, \infty)$ which do not satisfy the conditions \textit{1.}\ or \textit{2.}, it is $\lambdaC(\gamma) \in (0, \infty)$ and it $\lambda_C(\gamma)$ satisfies the critical point condition  $\partial_\lambda\limRiskArgC{\lambda}(\lambdaC(\gamma))=0$, or equivalently,
    \begin{align*}
        0=\lambdaC(\gamma)-\SNRstat^{-1}\gamma-\frac{\genConf}{2\gamma}\left(1+\lambdaC(\gamma)+\gamma-\sqrt{\varphi(\lambdaC(\gamma))}\right)\varphi(\lambdaC(\gamma))\,,
    \end{align*}
    where $\varphi(\lambda)=(1+\lambda+\gamma)^2-4\gamma$.
\end{enumerate}
\end{lemma}
\begin{proof}
We use the shorthand notation $\varphi(\lambda)=(1+\lambda+\gamma)^2-4\gamma$. Recall the confounding strength $\genConf=(r^2+\eta) / \signalstat$ and the statistical signal-to-noise ratio $\SNRstat=\signalstat / \sigmastat$. 
The derivative of the limiting causal risk $\limRiskArgC{\lambda}$ in $\lambda$ is given by
\begin{align*}
    \partial_\lambda\limRiskArgC{\lambda} = \frac{2\signalstat}{\varphi(\lambda)^{3/2}}\left(\lambda-\SNRstat^{-1}\gamma-\frac{\genConf}{2\gamma}\left(1+\lambda+\gamma-\sqrt{\varphi(\lambda)}\right)\varphi(\lambda)\right)
\end{align*}
\begin{enumerate}
    \item The first condition $\partial_\lambda\limRiskArgC{\lambda}\geq 0$ for all $\lambda>0$ can be equivalently rearranged for the confounding strength as
    \begin{align*}
        \genConf \leq 2\gamma\frac{\lambda-\SNRstat^{-1}\gamma}{\left(1+\lambda+\gamma-\sqrt{\varphi(\lambda)}\right)\varphi(\lambda)}=f(\lambda, \gamma, \SNRstat)\,,
    \end{align*}
    where $f$ is the function investigated in Lemma~\ref{lem:properties_func_f}. This in turn is equivalent to taking the infimum over $\lambda$, which is given by Lemma~\ref{lem:properties_func_f} as
    \begin{align*}
        \genConf\leq\inf_{\lambda>0}f(\lambda, \gamma,\SNRstat)=-\SNRstat^{-1}\frac{\gamma\max\{1,\gamma\}}{(1-\gamma)^2}.
    \end{align*}
    Note that for $\gamma=1$ this infimum is $-\infty$, so the condition cannot be satisfied for any $\genConf$.
    \item The proof of the second claim is analogue to the first with the reverse inequality $\partial_\lambda\limRiskArgC{\lambda}\leq 0$. Rearranging for $\genConf$ and using Lemma~\ref{lem:properties_func_f} yields the equivalent condition
    \begin{align*}
        \genConf\geq\sup_{\lambda>0}f(\lambda, \gamma,\SNRstat)=1\,.
    \end{align*}
    \item For the third claim, assume that the pair of $\genConf$ and $\gamma$ satisfies neither of the first points. We will use this to show that the derivative at 0 is negative $\partial_\lambda\limRiskArgC{\lambda}(0)<0$ and the derivative $\partial_\lambda\limRiskArgC{\lambda}$ for sufficiently large $\lambda$ is positive. This together then implies that the minimum $\lambdaC(\gamma)$ of the function $\limRiskArgC{\lambda}$ is indeed attained at a finite value in $(0,\infty)$, and $\limRiskArgC{\lambda}$ satisfies the critical point condition $\partial_\lambda\limRiskArgC{\lambda}(\lambdaC(\gamma))=0$.
    
    For the derivative at 0, assume that the converse is true, that is, $\partial_\lambda\limRiskArgC{\lambda}(0)\geq0$. Rearranging this condition for $\genConf$ yields similarly to the first case of this lemma that $\genConf\leq f(0, \gamma,\SNRstat)$. However Lemma~\ref{lem:properties_func_f} states that $f$ is increasing in $\lambda$, which means that this condition already implies $\genConf\leq f(\lambda, \gamma,\SNRstat)$ for all $\lambda$. This means that the pair $\genConf, \gamma$ would satisfy the condition of the first case, which contradicts our assumption.
    
    For the behavior of large $\lambda$, observe that the sign of the derivative is determined by the sign of the term $\lambda-\SNRstat^{-1}\gamma-\frac{\genConf}{2\gamma}\left(1+\lambda+\gamma-\sqrt{\varphi(\lambda)}\right)\varphi(\lambda)$. As derived in the proof of Lemma~\ref{lem:properties_func_f}, we have the asymptotic behavior
    \begin{align*}
        \left(1+\lambda+\gamma-\sqrt{\varphi(\lambda)}\right)\varphi(\lambda) = 2\gamma\lambda + \mathcal{O}(1)\,,
    \end{align*}
    which yields
    \begin{align*}
        \lambda-\SNRstat^{-1}\gamma-\frac{\genConf}{2\gamma}\left(1+\lambda+\gamma-\sqrt{\varphi(\lambda)}\right)\varphi(\lambda) = (1-\genConf)\lambda+\mathcal{O}(1)\,.
    \end{align*}
    Since the pair $\genConf,\gamma$ does by assumption not satisfy the conditions of the second case, we have $\genConf<1$, which means that the above term is eventually positive.
\end{enumerate}

\end{proof}

\NegativeCausalReg*
\begin{proof}
The first statement of the theorem is a special case of Theorem~\ref{thm:conf_incr_reg}. The necessary and sufficient condition for $\lambdaC=0$ stated there is equivalently reformulated as
\begin{align*}
    \genConf&\leq -\SNRstat^{-1}\frac{\gamma \max{\{1,\gamma\}}}{(1 - \gamma)^2} \\
    \Leftrightarrow\hspace{50pt} -\SNRstat\genConf &\geq \frac{\gamma \max{\{1,\gamma\}}}{(1 - \gamma)^2}\\
    \Leftrightarrow\hspace{27pt} \SNRcaus-\SNRstat &\geq \frac{\gamma \max{\{1,\gamma\}}}{(1 - \gamma)^2}\,,
\end{align*}
where the last part used the equality $\SNRcaus=(1-\genConf)\SNRstat$.
The statement about negative $\lambdaC$ refers to the fact that the derivative of the risk at 0 can be positive, that is, $\partial\limRiskArgC{\lambda}(0)>0$. This was shown in the proof of Lemma~\ref{lem:optimal_causal_lambda_characterization} and suggests that without our restriction $\lambdaC\geq 0$, a negative value of $\lambda$ would yield an even smaller risk.

For the second statement, observe that the condition $\genConf>0$ implies the cases 2. or 3. from Lemma~\ref{lem:optimal_causal_lambda_characterization}. In particular, this implies $\lambdaC>0$. The proof of Lemma~\ref{lem:optimal_causal_lambda_characterization} showed that in both of these cases it holds $\partial_\lambda\limRiskArgC{\lambda}(0)<0$, which means that the causal limiting risk $\limBiasArgC{\lambda}$ is strictly decreasing in a small neighborhood around 0. In particular, this implies that the minimal risk is strictly smaller than the risk at 0, that is, $\limRiskArgC{0}>\limRiskArgC{\lambdaC}$.

\end{proof}

\GenStrongRegCausVsStat*
\begin{proof}
Lemma~\ref{lem:optimal_causal_lambda_characterization} distinguishes between three different regimes of $\genConf$. The first two regimes yield
\begin{align*}
    \genConf\leq -\SNRstat^{-1}\frac{\gamma \max{\{1,\gamma\}}}{(1 - \gamma)^2}
    \implies
    \lambdaC=0
    \quad\text{and}\quad
    1\leq\genConf\implies\lambdaC=\infty\,.
\end{align*}
Combined with $\lambdaS=\SNRstat^{-1}\gamma\in(0,\infty)$, these regimes agree with the claim in the theorem. It remains to show that the theorem also holds for the last regime $-\SNRstat^{-1}\frac{\gamma \max{\{1,\gamma\}}}{(1 - \gamma)^2}<\genConf<1$. In this regime according to Lemma~\ref{lem:optimal_causal_lambda_characterization}, the optimal causal regularization $\lambdaC$ satisfies the critical point condition
\begin{align*}
0&=\lambdaC-\SNRstat^{-1}\gamma-\frac{\genConf}{2\gamma}\left(1+\lambdaC+\gamma-\sqrt{\varphi(\lambdaC)}\right)\varphi(\lambdaC)\\
\Leftrightarrow\quad \lambdaC-\lambdaS &= \frac{\genConf}{2\gamma}\left(1+\lambdaC+\gamma-\sqrt{\varphi(\lambdaC)}\right)\varphi(\lambdaC)\,.
\end{align*}
Since the term $1/(2\gamma)\left(1+\lambdaC+\gamma-\sqrt{\varphi(\lambdaC)}\right)\varphi(\lambdaC)$ is positive, the sign of $\lambdaC-\lambdaS$ is determined by the sign of $\genConf$ as claimed in the theorem.
\end{proof}

\ConfIncrReg*
\begin{proof}
The theorem follows directly from Lemma~\ref{lem:optimal_causal_lambda_characterization}, except for the statement about $\lambdaC$ being strictly increasing in $\genConf$. In the corresponding regime, Lemma~\ref{lem:optimal_causal_lambda_characterization} states that $\lambdaC$ satisfies the critical point condition $\partial_\lambda\limRiskArgC{\lambda}(\lambdaC)=0$, which we will use to show that the derivative of $\lambdaC$ in $\genConf$ is strictly positive.
For readability, we use the notation $x(\genConf)=1+\lambdaC(\genConf)+\gamma$ and  $\varphi(\zeta)=x(\zeta)^2-4\gamma$. The optimal causal regularization $\lambdaC(\genConf)$ satisfies the critical point condition 
\begin{align*}
    0=x(\genConf)-(1+\gamma+\SNRstat^{-1}\gamma)-\frac{\genConf}{2\gamma}\left(x(\genConf)-\sqrt{\varphi(\genConf)}\right)\varphi(\genConf)\eqqcolon g(x(\genConf), \genConf)\,.
\end{align*}
Rearranging this equation yields
\begin{align}\label{eq:critical_rearranged}
    \frac{\genConf}{2\gamma}\left(x(\genConf)-\sqrt{\varphi(\genConf)}\right)=\frac{x(\genConf)-(1+\gamma+\SNRstat^{-1}\gamma)}{\varphi(\genConf)}\,.
\end{align}
The partial derivatives of the function $g=g(x,\genConf)$ evaluated at $(x(\genConf),\genConf)$ are given by
\begin{align*}
    \partial_\genConf g(x(\genConf),\genConf) 
    =-\frac{1}{2\gamma}\left(x(\genConf)-\sqrt{\varphi(\genConf)}\right)\varphi(\genConf)<0
\end{align*}
and
\begin{align*}
    \partial_x g(x(\genConf),\genConf)&=1-\frac{\genConf}{2\gamma}\left[\left(1-\frac{x(\genConf)}{\sqrt{\varphi(\genConf)}}\right)\varphi(\genConf)+2x(\genConf)\left(x(\genConf)-\sqrt{\varphi(\genConf)}\right)\right]\\
    &=1-\frac{\genConf}{2\gamma}\left(x(\genConf)-\sqrt{\varphi(\genConf)}\right)\left(2x(\genConf)-\sqrt{\varphi(\genConf)}\right)\\
    &=1-\frac{x(\genConf)-(1+\gamma+\SNRstat^{-1}\gamma)}{\varphi(\genConf)}\left(2x(\genConf)-\sqrt{\varphi(\genConf)}\right)\tag{Using Eq.~\eqref{eq:critical_rearranged}}\\
    &>1-\frac{x(\genConf)-2\sqrt{\gamma}}{\varphi(\genConf)}\left(2x(\genConf)-\sqrt{\varphi(\genConf)}\right)\tag{$1+\gamma+\SNRstat^{-1}\gamma>2\sqrt{\gamma}$}\,.
\end{align*}
Since $\varphi(\genConf)=(x(\genConf)-2\sqrt{\gamma})(x(\genConf)+2\sqrt{\gamma})<(x(\genConf)+2\sqrt{\gamma})^2$, it further follows
\begin{align*}
    \partial_x g(x(\genConf),\genConf)
    &>1-\frac{x(\genConf)-2\sqrt{\gamma}}{(x(\genConf)-2\sqrt{\gamma})(x(\genConf)+2\sqrt{\gamma})}\left(2x(\genConf)-(x(\genConf)+2\sqrt{\gamma})\right)\\
    &=1-\frac{x(\genConf)-2\sqrt{\gamma}}{x(\genConf)+2\sqrt{\gamma}}\\
    &>0\,.
\end{align*}
With these results, we can take the derivative in $\genConf$ of the critical point condition $0=g(x(\genConf),\genConf)$ and obtain
\begin{align*}
    0=\frac{\diff}{\diff\genConf}g(x(\genConf), \genConf) 
    =\underbrace{\partial_x g(x(\genConf),\genConf)}_{>0} \cdot \frac{\diff x}{\diff\genConf}(\genConf) + \underbrace{\partial_\genConf g(x(\genConf),\genConf)}_{< 0} \cdot 1\,,
\end{align*}
which yields $0<\frac{\diff x}{\diff\genConf}(\genConf)=\frac{\diff \lambdaC}{\diff\genConf}(\genConf)$. This implies that $\lambdaC$ is increasing in $\genConf$ and concludes the proof.
\end{proof}

\end{document}

%% file: fig/causal_DAG.tex
\begin{tikzpicture}[x=1.5cm,y=1.0cm]
    % Nodes
    \node[obs] (x) {$x$} ; %
    \node[latent, above right=of x] (z) {$z$} ; %
    \node[obs, below right=of z] (y) {$y$} ; %
    \node[latent, above=of y] (eps) {$\varepsilon$} ; %

    % Edges
    \edge {x} {y}  ; %
    \edge {eps} {y} ; %
    \edge {z} {y} ; %
    \edge {z} {x} ; %
    
    % Label edges by redrawing invisible lines with regular tikz commands
    \path (x) -- (y) node [midway,above](TextNode){$\beta$};
    \path (z) -- (x) node [midway,above](TextNode){$M$};
    \path (z) -- (y) node [midway,above](TextNode){$\alpha$};
\end{tikzpicture}

%% file: fig/statistical_DAG.tex
\begin{tikzpicture}[x=1.5cm,y=1.0cm]
    % Nodes
    \node[obs] (x) {$x$} ; %
    \node[obs, right=of x] (y) {$y$} ; %
    \node[latent, above=of y] (eps) {$\tilde{\varepsilon}$} ; %

    % Edges
    \edge {x} {y} ; %
    \edge {eps} {y} ; %
    
    % Label edges by redrawing invisible lines with regular tikz commands
    \path (x) -- (y) node [midway,above](TextNode){$\tilde{\beta}$};
\end{tikzpicture}